\documentclass{article}

\usepackage{fullpage}
\usepackage{microtype}
\usepackage{graphicx}
\usepackage{booktabs} % for professional tables
\usepackage{amsmath, amssymb, amscd, amsthm, amsfonts}
\usepackage{graphicx}
\usepackage{hyperref}
\usepackage{booktabs}
\usepackage[dvipsnames]{xcolor}
\usepackage[square, numbers]{natbib}
\usepackage{titlesec}
\usepackage{tikz}
\usetikzlibrary{positioning}
\usepackage{algorithm,algorithmic}

\usepackage{subcaption}
\newtheorem{theorem}{Theorem}

\newtheorem{lemma}{Lemma}

\newtheorem{definition}{Definition}[section]
\newtheorem{example}{Example}

\begin{document}

% command to put ref in parentheses
\newcommand{\refp}[1]{(\ref{#1})}
  
%user defined commands
\newcommand*\Mirr[1]{\textsc{Mirr} (#1)}
\newcommand*\Wd[0]{\textit{W}_2}
\newcommand*\dd[2]{\frac{\partial #1}{\partial #2}}
\newcommand*\stein[1]{\mathcal{A}_{#1}}
\newcommand*\w[0]{\textbf{w}}
\newcommand*\at[2]{\left.#1\right|_{#2}}
\newcommand*\ps[0]{p^*}
\newcommand*\del[0]{\partial}
\newcommand*\R[0]{\mathbb{R}}
\newcommand*\Z[0]{\mathbb{Z}}
\renewcommand*\div[0]{\nabla \cdot}
\newcommand*\ddt[0]{\frac{d}{d t}}
\newcommand*\dds[0]{\frac{d}{d s}}
\newcommand*\ddp[1]{\frac{\delta #1}{\delta p}}
\newcommand*\tr[0]{\text{tr}}
\newcommand*\KL[2]{\mathcal{KL}\left(#1\|#2\right)}
\newcommand*\lin[1]{\langle #1\rangle}
\newcommand*\E[1]{\mathbb{E}\left[#1\right]}
\newcommand*\Ep[2]{\mathbb{E}_{#1}\left[#2\right]}
\newcommand*\grad[3]{\mathcal{D}_{#1}(#2,#3)}
\newcommand*\slope[1]{\|\mathcal{D}_{#1}\|}
\newcommand*\md[1]{\|\frac{d}{dt}#1\|}
\newcommand*\ap[2]{\tilde{#1}_{#2}}
\newcommand*\F[0]{\mathcal{F}}
\newcommand*\D[0]{\mathcal{D}}
\renewcommand*\H[0]{\mathcal{H}}
\newcommand*\Pspace[0]{\mathscr{P}}
\newcommand*\Uspace[0]{\mathbb{U}}
\renewcommand*\t[1]{\tilde{#1}}
\renewcommand*\d[0]{\text{d}}
\newcommand*\ot[1]{#1^{\perp}}
\newcommand*\proj[0]{\text{Proj}}
\newcommand\numberthis{\addtocounter{equation}{1}\tag{\theequation}}
\newcommand*\lrb[1]{\left[#1\right]}
\newcommand*\lrp[1]{\left(#1\right)}
\newcommand*\eu[1]{\left\| #1 \right\|}
\renewcommand*\div[0]{\nabla \cdot}
\renewcommand*\H[0]{\mathcal{H}}
\renewcommand*\t[1]{\tilde{#1}}
\renewcommand*\d[0]{\delta}
\newcommand*\xt[0]{\tilde{x}}
\newcommand*\vt[0]{\tilde{v}}
\newcommand*\pt[0]{\tilde{p}}
\newcommand*\qt[0]{\tilde{q}}
\newcommand*\Phit[0]{\tilde{\Phi}}
%NewSymbols
\def\ci{\perp\!\!\!\perp}
\def\half{\frac{1}{2}}
\def\lv{\lVert}
\def\rv{\rVert}
\def\ke{\mathcal{K}}
\def\kc{\hat{\mathcal{K}}_m}

\newcommand*\diff{\mathop{}\!\mathrm{d}}

\newcommand{\note}[1]{\textbf{{\color{red}#1}}}
\newcommand{\rem}[1]{{\emph{#1}}}

\newcommand{\argmin}{\operatornamewithlimits{argmin}}
\newcommand{\argmax}{\operatornamewithlimits{argmax}}
\newcommand*\samethanks[1][\value{footnote}]{\footnotemark[#1]}

\newcommand\psued[1]{R^s_{#1}}

\newcommand\regret[1]{R^c_{#1}}

\newcommand\ind[2]{\mathbb{I}_{#1,#2}}
\newcommand\indi[0]{\mathbb{I}}

\newcommand\algname{\mathsf{OSOM}}
\newcommand\UCB{\mathsf{UCB}}
\newcommand\OFUL{\mathsf{OFUL}}
\newcommand\corral{\mathsf{CORRAL}}
\newcommand\moss{\mathsf{MOSS}}
\newcommand\linucb{\mathsf{LinUCB}}
\newcommand\expalg{\mathsf{EXP}4}
\newcommand\expsim{\mathsf{EXP}3}
\newcommand\sao{\mathsf{SAO}}

\newcommand\tstar{\tau_*}
\newcommand\unitball{\mathbb{B}_2^d(1)}

\newcommand\mustar[1]{\mu_{#1}}
\newcommand\muhat{\widehat{\mu}}
\newcommand\nustar[1]{\nu_{#1}}
\newcommand\thetastar{\theta^{*}}

\newcommand\tinrange[2]{t \in \{#1,\ldots, #2\}}
\newcommand\sinrange[2]{s \in \{#1,\ldots, #2\}}
\newcommand\iinrange[2]{i \in \{#1,\ldots, #2\}}
\newcommand\Xt{X_{t}}
\newcommand\Xs{X_{s}}
\newcommand\Xit[2]{X_{#1,#2}}
\newcommand\alpt[2]{\alpha_{#1,#2}}
\newcommand\betat[2]{\beta_{#1,#2}}

\newcommand\muit[2]{\tilde{\mu}_{#1,#2}}

\newcommand\gis{g_{i,s}}
\newcommand\git{g_{i,t}}

\newcommand\ccomplex{\mathcal{C}_t^c}
\newcommand\csimple{\mathcal{C}_t^s(i)}

\newcommand\mconstraint{V_t}

\newcommand\Omegabold{\boldsymbol{\Omega}}
\newcommand\Ostar{\Omegabold^*}
\newcommand\Otilde{\tilde{\Omegabold}_t}
\newcommand\Ohat{\hat{\Omegabold}_t}

\newcommand\Tstar{\boldsymbol{\theta}^*}
\newcommand\Ttilde{\tilde{\boldsymbol{\theta}}_t}
\newcommand\thetabold{\boldsymbol{\theta}}
\newcommand\That{\hat{\theta}_t}
\newcommand\Deltabold{\boldsymbol{\Delta}}

\newcommand\Xboldt{\mathbf{X}_{1:t}}
\newcommand\Aboldt{\boldsymbol{\alpha}_{K+1:t}}
\newcommand\Gboldt{\mathbf{G}_{K+1:t}}

\newcommand\Xboldtn{\mathbf{X}}
\newcommand\Aboldtn{\boldsymbol{\alpha}}
\newcommand\Gboldtn{\mathbf{G}}
\newcommand\etabold{\boldsymbol{\eta}}
\newcommand\mubold{\boldsymbol{\mu}}
\newcommand\mutbold{\boldsymbol{\tilde{\mu}}}

\newcommand\averagerew{\bar{g}_{i,t}}

\newcommand\Otildes{\tilde{\Omega}_s}
\newcommand\Ohats{\hat{\Omega}_s}

\newcommand\tmin{\tau_{\min}(\delta',T)}
\newcommand\updela{\Upsilon_{\delta'}(t,T)}
\newcommand\mone{\mathcal{M}_{\delta'}(t)}
\newcommand\mtwo{\mathcal{M}_2}
\newcommand\kenv{\mathcal{K}_{\delta'}(t,T)}
\newcommand\kenvs{\mathcal{K}_{\delta'}(s-1,T)}
\newcommand\kenvt{\mathcal{K}_{\delta'}(t,T)}
\newcommand\wenv{\alpha_{\delta'}(t,T)}
\newcommand\wenvone{\alpha_{\delta'}(t,T)}
\newcommand\wenvtau{\alpha_{\delta'}(\tstar-1,T)}
\newcommand\wenvn{\alpha_{\delta'}(T,T)}
\newcommand\wenvnre{\alpha_{\delta'}(T,T)}
\newcommand\bfm[1]{\mathbf{#1}}
\newcommand\loweig{\gamma}
\newcommand\higheig{\rho_{\max}}

\newcommand\Oh{\mathcal{O}}

%%%%NEW MACROS FOR NEW IDEAS!
\newcommand\Rmat{\boldsymbol{R}}
\newcommand\Rhat{\widehat{\boldsymbol{R}}}
\newcommand\Rgamma{[\boldsymbol{R}]_{\gamma}}
\newcommand\Dhat{\widehat{D}}
\newcommand\Sigmabold{\boldsymbol{\Sigma}}
\newcommand\Lambdabold{\boldsymbol{\Lambda}}
\newcommand\Ubold{\boldsymbol{U}}
\newcommand\Sigmahat{\widehat{\Sigmabold}}
\newcommand\Sigmabar{\overline{\Sigmabold}}
\newcommand\Omegahat{\widehat{\Omegabold}}
\newcommand\Sigmagamma{[\Sigmabold]_{\gamma}}
\newcommand*\Tgamma[1]{T_{\gamma}\left(#1\right)}
\newcommand\Uhat{\widehat{U}}
\newcommand\Lambdahat{\widehat{\Lambda}}
\newcommand\lambdahat{\widehat{\lambda}}
\newcommand\Sgaphat{\widehat{\mathcal{E}}}
\newcommand\Sgaptilde{\widetilde{\mathcal{E}}}
\newcommand\Sgap{\mathcal{E}}
\newcommand\Sgapbar{\overline{\mathcal{E}}}
\newcommand\tildex{\widetilde{\x}}
\newcommand\Ecal{\mathcal{E}}
\newcommand\Acal{\mathcal{A}}
\newcommand\tauswitch{\tau_{\mathsf{switch}}}
\newcommand\mutilde{\widetilde{\mu}}

\newcommand\eye{\boldsymbol{I}_d}

\newcommand\x{\boldsymbol{x}}
\newcommand\y{\boldsymbol{y}}
\newcommand\M{\boldsymbol{M}}
\newcommand\UM{\boldsymbol{U_M}}
\newcommand\xibold{\boldsymbol{\xi}}
\newcommand\op{\mathsf{op}}

\newcommand\T{\mathcal{T}}
\newcommand\W{\mathcal{W}}

\newcommand\vm[1]{\textcolor{green}{VM -- #1}}
\newcommand\ak[1]{\textcolor{red}{AK -- #1}}

\begin{center}

{\bf{\Large{Universal and data-adaptive algorithms for model selection in linear contextual bandits}}}

\vspace*{.2in}

{\large{
\begin{tabular}{ccc}
  Vidya Muthukumar$^{\dagger}$ & Akshay Krishnamurthy$^{\ddagger}$ \\
\end{tabular}}}

\vspace*{.2in}

\begin{tabular}{c}
Electrical and Computer Engineering and Industrial and Systems Engineering, Georgia Institute of Technology$^{\dagger}$ \\
Microsoft Research, New York City $^{\ddagger}$ \\
\end{tabular}

\vspace*{.2in}
\date{}

\end{center}

% You may provide any keywords that you
% find helpful for describing your paper; these are used to populate
% the "keywords" metadata in the PDF but will not be shown in the document

% this must go after the closing bracket ] following \twocolumn[ ...

% This command actually creates the footnote in the first column
% listing the affiliations and the copyright notice.
% The command takes one argument, which is text to display at the start of the footnote.
% The \icmlEqualContribution command is standard text for equal contribution.
% Remove it (just {}) if you do not need this facility.
%\printAffiliationsAndNotice{}  % leave blank if no need to mention equal contribution
%\printAffiliationsAndNotice{\icmlEqualContribution} % otherwise use the standard text.

\begin{abstract}
    Model selection in contextual bandits is an important complementary problem to regret minimization with respect to a fixed model class.
  We consider the simplest non-trivial instance of model-selection: distinguishing a simple multi-armed bandit problem from a linear contextual bandit problem.
  Even in this instance, current state-of-the-art methods explore in a suboptimal manner and require strong ``feature-diversity'' conditions.
  In this paper, we introduce new algorithms that a) explore in a data-adaptive manner, and b) provide model selection guarantees of the form $\mathcal{O}(d^{\alpha} T^{1- \alpha})$ with \emph{no} feature diversity conditions whatsoever,
  where $d$ denotes the dimension of the linear model and $T$ denotes the total number of rounds.
The first algorithm enjoys a ``best-of-both-worlds'' property, recovering two prior results that hold under distinct distributional assumptions, simultaneously.
 The second removes distributional assumptions altogether, expanding the scope for tractable model selection. 
  Our approach extends to model selection among nested linear contextual bandits under some additional assumptions.
\end{abstract}
%!TEX root = main.tex

\section{Introduction}\label{sec:intro}

The \emph{contextual bandit (CB) problem}~\citep{woodroofe1979one,langford2008epoch,li2010contextual} is a foundational paradigm for online decision-making.
In this problem, the decision-maker takes one out of $K$ actions as a function of available contextual information, where this function is chosen from a fixed policy class and is typically learned from past outcomes.
%% Throughout this paper, we will consider $K$ to be a small constant.
%% \vm{is this the right place for this? i'm not 100\% happy}
Most work on CB has centered around designing algorithms that minimize regret with respect to the best policy in hindsight; a particular non-triviality involves doing this in a computationally efficient manner~\citep{agarwal2014taming,syrgkanis2016efficient,foster2018contextual}.
When the rewards are \emph{realizable} under the chosen policy class, this is now an essentially solved problem~\citep{foster2020beyond,simchi2020bypassing}.
% \vm{Feel free to tone down language here.}

A complementary problem to regret minimization with respect to a fixed policy class is \emph{choosing the policy class} that is best for the problem at hand.
This constitutes a model selection problem, and its importance is paramount in CB, as selecting a class that either underfits or overfits can lead to highly suboptimal performance.
To see why, consider the simplest model selection instance, which involves deciding whether to use the contexts (with, say, a policy class of $d$-dimensional linear functions) or simply run a multi-armed bandit (MAB) algorithm.
Making this choice a priori is suboptimal one way or another: if we choose a MAB algorithm, 
we obtain the optimal $\Oh(\sqrt{KT})$ regret to the best fixed action, but the latter may be highly suboptimal if the rewards depend on the context.
On the other hand, if we choose a linear CB algorithm, we incur $\Oh((\sqrt{K} + \sqrt{d})\sqrt{T})$ regret even when the rewards do not depend on the contexts due to overfitting, which is highly suboptimal\footnote{Throughout the paper, we consider $K$ to be a small constant and allow the feature dimension $d$ to be quite large. This is in contrast to the typical case of a \emph{linear non-contextual bandit} problem where the number of arms $K$ can be quite large, and the features model structure across arms to ensure a much smaller regret (see, e.g., Chapter 19 of~\citet{lattimore2018bandit}). There, a rather different model selection problem arises, where the linear bandit class is actually the simpler one. 
For the CB model selection problem, the $K\geq d$ regime renders the problem uninteresting: it is always better to use a linear CB algorithm as it will better model structure and incur negligible regret overhead.
}.
 %% --- this is highly suboptimal compared to the $\Oh(\sqrt{KT})$ guarantee that we could have obtained had we known about the simpler structure.
%% This example demonstrates that model selection must be \emph{data-driven}.
To avoid these two failure modes, we seek an algorithm that achieves the best-of-both-worlds, by retaining the linear CB guarantee
%% most complex regret guarantee 
while adapting to hidden structure if it exists.
Focusing on the MAB-vs-linear setting, 
the strongest variant of the model selection objective asks: % takes the following form:
\vspace{-2mm}
\begin{center}
\textit{\textbf{Objective 1.} Can we design a \textbf{single} algorithm that simultaneously achieves the respective minimax-optimal rates of $\Oh(\sqrt{KT})$ under simple MAB structure (when it exists) and $\Oh((\sqrt{K} + \sqrt{d})\sqrt{T})$ under $d$-dimensional linear CB structure?}
\end{center}
%
% \cite{chatterji2020osom} achieves Objective 1, but under strong assumptions that we elaborate on in Sections~\ref{sec:relatedwork} and~\ref{sec:setup}.
A related but weaker objective is also proposed in a COLT 2020 open problem~\citep{foster2020open}, restated here for the special case of linear CB:
\begin{center}
\textit{\textbf{Objective 2.} Can we design a \textbf{single} algorithm that simultaneously achieves the rate $\Oh(K^{\beta} T^{1 - \alpha})$ under simple MAB structure and $\Oh(K^{\beta} d^{\alpha} T^{1- \alpha})$ under linear CB structure, for some $\alpha \in (0,1/2]$ and some $\beta \in (0,1)$?}
\end{center}
% \vspace{-0.2cm}
%% Since we are considering $K$ to be a constant, we omit dependences on $K$ in the main paper.
%% \vm{not sure about this sentence? feel free to edit}
\citet{foster2020open} highlighted the importance of the $\widetilde{\Oh}(d^\alpha T^{1-\alpha})$-rate as verifying that model
selection is possible whenever the model class is
\emph{learnable}.
This is the case if and only if $d = o(T)$,
assuming, as we do, that $K$ is a small constant (we omit dependence
on $K$ in the main paper).

Even in the simplest instance of model selection between a multi-armed bandit and a linear contextual bandit, achieving either Objective 1 or Objective 2 under minimal assumptions remains open.
Existing approaches to address either the stronger Objective 1~\citep{chatterji2020osom} or the weaker Objective 2~\citep{foster2019model} make restrictive assumptions
% \footnote{We note here that the assumptions made across the two works are of varying strength: in particular, while~\cite{chatterji2020osom} assumes that the context corresponding to each action will be well-conditioned,~\cite{foster2019model} only requires well-condition on the distribution of the average of the contexts across actions. See Sections~\ref{sec:relatedwork} and~\ref{sec:priorart} for more details on the distinctions between the two approaches.} 
regarding the conditioning (what we will call \emph{diversity}) of the contexts.
Other, more data-agnostic approaches~\citep{agarwal2017corralling,pacchiano2020model,pacchiano2020regret,lee2021online} achieve neither of the above objectives.
This leads us to ask whether we can design a \emph{universal} model selection approach that is data-agnostic (other than requiring a probability model on the contexts) \emph{and} achieves either Objective 1 or 2.

Another important question is the \emph{adaptivity} of approaches to situations in which model selection is especially tractable.
At the heart of effective data-driven model selection is a meta-exploration-vs-exploitation tradeoff: while we need to exploit the currently believed simpler model structure, we also need to explore sufficiently to discover potential complex model structure.
Most approaches to model selection incorporate forced exploration of an $\epsilon$-greedy type to navigate this tradeoff; however, such exploration may not always be needed.
Indeed,~\cite{chatterji2020osom} use no forced exploration in their approach and thereby achieve the optimal guarantee of Objective 1, but their approach only works under restrictive diversity assumptions.
It is natural to ask whether we can design \emph{data-adaptive} exploration schedules that employ forced exploration only when it is absolutely needed, thus recovering the strongest possible guarantee (Objective 1) under favorable situations and a weaker-but-still-desirable guarantee (Objective 2) otherwise.

\subsection{Our contributions}\label{sec:contributions}

From the above discussion, it is clear that algorithm design for model selection that satisfies the criteria posed in~\cite{foster2020open} involves two non-trivial components: a) designing an efficient statistical test to distinguish between simple and complex model structure that works under minimal assumptions, and b) designing an exploration schedule to ensure that sufficiently expressive data is collected for a) to succeed.
In this paper, we advance the state-of-the-art for both of these components in the following ways:
% \vspace{-3mm}
\begin{itemize}
\item We design a new test based on eigenvalue thresholding that works for all stochastic sub-Gaussian contexts; in contrast to~\cite{chatterji2020osom} and~\cite{foster2019model}, it does not require any type of context diversity. 
We utilize the fact that ``low-energy'' directions can be thresholded and ignored to estimate the gap in error between model classes. 
See Theorem~\ref{thm:universal} for our new model selection guarantee, which only requires stochasticity on the contexts to meet Objective 2.
% \vspace{-3mm}
\item We also design a data-adaptive exploration schedule that performs forced exploration only when necessary.
This approach meets Objective 2 under the \emph{action-averaged} feature diversity assumption that is made in~\cite{foster2019model}, but also the stronger Objective 1 under the stronger assumption that the feature of \emph{each action} is diverse (made in~\cite{chatterji2020osom}). 
In fact, our approach not only meets Objective 1 under \emph{action-specific} feature diversity, but even the instance-optimal rate (in line with the results of~\cite{chatterji2020osom}).
See Theorem~\ref{thm:dataadaptive} for a precise statement of our new adaptive guarantee on model selection.
% \ak{The ``in general'' part is not quite correct right?} \vm{Yes. I've amended to be a bit more precise (but maybe it's too in-the-weeds now).}
\end{itemize}
% \vspace{-3mm}
Taken together, our results advance our understanding of model
selection for contextual bandits, by demonstrating how statistical
approaches can yield universal (i.e., nearly assumption-free) and
adaptive guarantees.

\subsection{Related work}\label{sec:relatedwork}

While model selection is a central and classical topic in machine
learning and statistics, most results primarily apply to supervised
offline and full-information online learning. Only recently has
attention turned to model selection in online partial information
settings including contextual bandits. Here, we focus on this growing
body of work, which we organize based on overarching algorithmic principles. 

\paragraph{Corralling approaches.} 
The first line of work constitutes a hierarchical learning scheme
where a meta-algorithm uses bandit techniques to compete with many
(contextual) bandit algorithms running as base learners. One of the
first such approaches is the \textsc{Corral} algorithm
of~\citet{agarwal2017corralling}, which uses online mirror descent
with the log-barrier regularizer as the meta-algorithm. Subsequent
work focuses on adapting \textsc{Corral} to the stochastic
setting~\citep{pacchiano2020regret,lee2021online} and developing
UCB-style
meta-algorithms~\citep{cutkosky2020upper,arora2021corralling}. These
approaches are quite general and can often be used with abstract
non-linear function classes.
However, they do not meet either of Objectives 1 or 2 in general.
In our setting, these approaches yield the tuple of rates $(\widetilde{\Oh}(\sqrt{T}),\widetilde{\Oh}(d\sqrt{T}))$, which clearly cannot be expressed in the form $(\widetilde{\Oh}(T^{1 - \alpha}), \widetilde{\Oh}(d^{\alpha}T^{1-\alpha})$ for any value of $\alpha \in (0,1)$.
Consequently, the problem of model selection as described in~\cite{foster2020open} is left open even for linear classes.
% \vspace{-3mm}
\paragraph{Statistical approaches.} 
The second line of algorithmic approaches involves constructing
statistical tests for model misspecification. This approach was
initially used in the context of model selection concurrently
by~\citet{foster2019model} and~\citet{chatterji2020osom}, who focus on
the linear setting. At a high level, these papers develop efficient
misspecification tests under certain covariate assumptions and use
these tests to obtain $d^\alpha T^{1-\alpha}$-style model selection
guarantees. In particular,~\citet{foster2019model} use a ``sublinear''
square loss estimator under somewhat mild covariate assumptions to
obtain $d^{1/3}T^{2/3}$ regret, while~\citet{chatterji2020osom} obtain
$\sqrt{dT}$ regret under stronger covariate assumptions.  As these two
works are the foundation for our results, we discuss these papers in
detail in the sequel.

Several recent papers extend statistical testing approaches in several
ways. ~\citet{ghosh2021problem} estimate the support of the parameter
vector, which fundamentally incurs a dependence on the magnitude of
the smallest non-zero coefficient. Beyond the linear
setting,~\citet{cutkosky2021dynamic} use the ``putative'' regret bound
for each model class directly to test for misspecification,
while~\citet{ghosh2021model,krishnamurthy2021optimal} consider general
function classes with realizability. While these latter approaches are
more general than ours, they cannot be directly used to obtain our
results. Indeed, central to our results (and those
of~\citet{foster2019model}) is the fact that our statistical test
provides a fast rate for detecting misspecification; this guarantee is
quantitatively better than what is provided by the putative regret
bound, requires carefully adjusting the exploration schedule, and is not available for general function classes. 
%% . At the same time, obtaining the fast rate guarantee requires
%% carefully adjusting the exploration schedule, so it cannot be used in
%% black-box fashion in the algorithm
%% of~\citet{krishnamurthy2021optimal}.

We also briefly mention two peripherally related lines of work. The first
is on representation selection in bandits and reinforcement
learning~\citep{papini2021leveraging,zhang2021provably}, which
involves identifying a feature mapping with favorable properties from
a small class of candidate mappings. While this is reminiscent of the model selection problem, the main differences are
that in representation selection all mappings are of the same
dimensionality and realizable, and the goal is to achieve much faster
regret rates by leveraging additional structural properties of the
``good'' representation.

The second line of work is on Pareto optimality in \emph{non-contextual} bandits and related
problems. Beginning with the result of~\citet{lattimore2015pareto},
these results show that certain non-uniform regret guarantees are not
achievable in various bandit settings. For
example,~\cite{lattimore2015pareto} shows that, in $K$-armed bandit
problems, one cannot simultaneously achieve $O(\sqrt{T})$ regret to
one specific arm, while guaranteeing $O(\sqrt{KT})$ regret to the
rest. Such results have been extended to both linear and Lipschitz non-contextual bandits~\citep{zhu2021pareto,locatelli2018adaptivity,krishnamurthy2019contextual}
as well as Lipschitz contextual bandits under margin
assumptions~\citep{gur2021smoothness}, and they establish that model
selection is not possible in these settings. However, these ideas have
not been extended to standard contextual bandit settings, which is our
focus.

\section{Setup}\label{sec:setup}

%% We now set up basic notation and the model selection problem that we will study.

\paragraph{Notation.}
We use boldface to denote vectors and matrices (e.g. $\boldsymbol{x}$ to denote a vector, and $x$ to denote a scalar).
For any value of $M < \infty$, $[M]$ denotes the finite set $\{1,\ldots,M\}$.
We use $\eu{\cdot}_2$ to denote the $\ell_2$-norm of a vector, and $\eu{\cdot}_{\mathsf{op}}$ to denote the operator norm of a matrix.
We use $\lin{\x,\y}$ to denote the Euclidean inner product between vectors $\boldsymbol{x}$ and $\boldsymbol{y}$.
We use big-Oh notation in the main text; $\widetilde{\Oh}(\cdot)$ hides dependences on the number of actions $K$, and $\Oh_{\delta}(\cdot)$ denotes a bound that hides a $\log\left(\frac{1}{\delta}\right)$ factor and holds with probability at least $1 - \delta$.

\subsection{The bandit-vs-contextual bandit problem}
The simplest instance of model selection involves a $d$-dimensional \emph{linear contextual bandit} problem with possibly hidden \emph{multi-armed bandit} structure.
This model was proposed as an initial point of study in~\cite{chatterji2020osom}.
Concretely, $K$ actions (which we henceforth call \emph{arms}) are available to the decision-maker at every round, and $T$ denotes the total number of rounds.
At round $t$, the reward of each arm is given by $G_{i,t} = \mu_i + \lin{\x_{i,t},\Tstar} + W_{i,t}$, where $\mu_i$ denotes the bias of arm $i$, $\x_{i,t} \in \R^d$ denotes the $d$-dimensional context corresponding to arm $i$ at round $t$, and $W_{i,t}$ denotes random noise.
Finally, $\Tstar \in \R^d$ denotes an unknown parameter.
We make the following standard assumptions on the problem parameters.
% \vspace{-3mm}
\begin{itemize}\itemsep1pt
\item The biases $\{\mu_i\}_{i=1}^K$ are assumed to be bounded between $-1$ and $1$.
\item The unknown parameter is assumed to be bounded, i.e. $\eu{\Tstar}_2 \leq 1$.
\item The contexts corresponding to each arm $i$ are assumed to be iid across rounds $t \geq 1$, and $1$-sub-Gaussian. We denote by $\Sigmabold_i$ the covariance matrix of the context $\x_{i,t}$, and additionally note that $\Sigmabold_i \preceq \eye$ as a consequence of the $1$-sub-Gaussian assumption. Without loss of generality (since bias can be incorporated into $\mu_i$), we assume that for each arm $i \in [K]$ the mean of the context $\x_{i,t}$ is equal to the zero vector.
% \footnote{This is without loss of generality, since bias can be incorporated into $\mu_i$.}.
\item The noise $W_{i,t}$ is iid across arms $i \in [K]$ and rounds $t \in [T]$, centered, and $1$-sub-Gaussian.
\end{itemize}
%
% \vspace{-3mm}
We denote the achieved pseudo-regret with respect to the best fixed arm (the standard metric for a MAB problem) by $R_T^S$, and the achieved pseudo-regret with respect to the best policy under a $d$-dimensional model (the standard metric for a linear CB problem) by $R_T^C$.
Notice that in the special case when $\Tstar = \boldsymbol{0}$, this reduces to a standard multi-armed bandit (MAB) instance.
The best possible regret rate is then given by $R_T^S = \Oh(\sqrt{KT})$ in the worst case, and we also have the instance-dependent rate $R_T^S = \Oh\left(\sum_{i \neq i^*} \frac{\log T}{\Delta_i}\right)$, where $\Delta_i := \mu^* - \mu_i$.
Both of these are known to be information-theoretically optimal~\citep{Lai:1985:AEA:2609660.2609757,audibert2009minimax}.
On the other hand, the minimax-optimal rate for the linear contextual bandit (linear CB) problem is given by $R_T^C = \Oh((\sqrt{d} + \sqrt{K}) \sqrt{T})$~\citep{chu2011contextual,abbasi2011improved}.
The following natural dichotomy in algorithm choice presents itself:
% \vspace{-3mm}
\begin{enumerate}
\item While the state-of-the-art for the linear CB problem achieves the minimax-optimal rate $R_T^C = \Oh((\sqrt{d} + \sqrt{K}) \sqrt{T})$, it does not adapt automatically to the simpler MAB case.
In particular, the regret $R_T^S$ will still scale with the dimension $d$ of the contexts owing both to unnecessary exploration built into linear CB algorithms and overfitting effects.
%\vm{This is actually a pretty subtle thing: it is not always clear that linear CB overexplores (see~\cite{hao2020}). But the overfitting effect is there for sure. So I've hedged a bit and said both.}
This precludes achieving the minimax-optimal rate of $R_T^S = \Oh(\sqrt{KT})$ in the MAB setting, let alone the instance-dependent rate.
\item On the other hand, any state-of-the-art algorithm that is tailored to the MAB problem would not achieve any meaningful regret rate for the linear CB problem, simply because it does not incorporate contextual information into its decisions.
\end{enumerate}
% \vspace{-3mm}
The simulations in~\cite{chatterji2020osom} empirically illustrate this dichotomy and clearly motivate the model selection problem in its most ambitious form, i.e. Objective 1 as stated in Section~\ref{sec:intro}.
Objective 2 constitutes a weaker variant of the model selection problem that was proposed in~\cite{foster2019model,foster2020open} and justified by the fact that it yields non-trivial model selection guarantees whenever the underlying class is learnable. 
While Objective 2 is in itself a desirable and non-trivial model selection guarantee, we note that it is strictly weaker than Objective 1.
To see this, note that the objectives coincide for $\alpha=1/2$, and since we require $d < T$ for sublinear regret in the first place, the rate $d^{\alpha} T^{1-\alpha}$ is a decreasing function in $\alpha$.
\\
\\
\\
\\
% \vm{I liked this updated para: short and to the point. I only modified the last sentence a bit; I thought it was a bit clunky earlier.}

\subsection{Meta-algorithm and prior instantiations}
\label{sec:priorart}

\begin{algorithm}[tb]
% \SetAlgoLined
\caption{Model selection meta-algorithm through one-shot sequential testing. 
$\Sgaphat$ denotes an estimator of the square-loss gap between the MAB and linear CB model, $\nu_t \in (0,1)$ denotes a forced-exploration parameter, and $\delta \in (0,1)$ denotes a failure probability. }
% $\W(t)$ denotes the set of designated exploration rounds (which is a subset of $[t]$).
% $\alpha_{\delta'}(s,t)$ denotes a testing radius set as a function of $s$ labeled samples and $t$ unlabeled samples, and will in general vary as a function of the estimator and theexploration schedule.
% \vm{This is a little clunky formatting-wise, but it does clean up the algorithm description below. Would be nice if we could have a caption at the bottom instead of the top.}
\label{alg:metamethod}
\begin{algorithmic}
\FOR{$t=1,\ldots,K$}
  \STATE Play arm $t$ and receive reward $g_{t,t}$,
\ENDFOR \\
\STATE ${\texttt{Current Algorithm}} \gets \text{`MAB'}$ \;
\FOR{$t=K+1,\ldots,n$}
\STATE $\textit{MAB Algorithm: } i_t = \text{ arm pulled by \textsc{UCB} }$ \\
\STATE  $\textit{CB Algorithm: } j_t = \text{ arm pulled by \textsc{LinUCB}}$\\
    \IF{${\texttt{Current Algorithm} = \text{`MAB'}}$} 
    \STATE Estimate square loss gap $\Sgaphat$ and \emph{declare misspecification} if $\Sgaphat > \alpha_{\delta}$ (where $\alpha_{\delta}$ is a threshold defined as a function of the filtration $\H_{t-1}$ and the failure probability $\delta$, and is specified for various algorithm choices in Appendices~\ref{sec:thm1proof} and~\ref{sec:thm2proof}). \\
    % \ak{Should we say threshold here? It's overloaded with eigenvalue thresholding, which we are not discussing at this point.} \vm{Agreed. I've amended, but I'm still not happy with the impreciseness of this}\\
    \STATE If misspecification detected, then set ${\texttt{Current Algorithm}} \leftarrow \text{`CB'}$.\\
    \ENDIF
    \IF{${\texttt{Current Algorithm}}=\text{`MAB'}$}
    \STATE Select $U_t = 1$ with probability $1 - \nu_t$, $U_t = 0$ otherwise.\\
    \STATE Play arm $A_t = i_t$ if $U_t = 1$ and $A_t \sim \text{Unif}[K]$ if $U_t = 0$.\\
    \STATE Receive reward $g_{A_t,t}$. \\
    \ELSIF{${\texttt{Current Algorithm}}=\text{`CB'}$} 
    \STATE Play arm $j_t$ and receive $g_{j_t,t}.$
    \ENDIF
    % \STATE Pass rewards to current algorithm.
\ENDFOR
\end{algorithmic}
\end{algorithm}
%
% \vspace{-3mm}
As mentioned in Section~\ref{sec:relatedwork}, the vast toolbox of corralling-type approaches does not achieve either Objective 1 or 2 for model selection.
~\cite{chatterji2020osom} and~\cite{foster2019model}, which are concurrent to each other, are among the first approaches to tackle the model selection problem and the only ones that achieve Objectives 1 and 2 respectively---but under additional strong assumptions.
Both approaches use the same structure of a statistical test to distinguish between a simple (MAB) and complex (CB) instance.
This meta-approach is described in Algorithm~\ref{alg:metamethod}. 
% \todo{Improve formatting of algorithm in double-column format: some unnecessary new lines}
Here, $A_t$ denotes the arm that is pulled at round $t$, and as is standard in bandit literature, $\H_t := \{A_s,G_{A_s,s}\}_{s=1}^t$ is the relevant filtration at round $t$.

As our results also involve instantiating this meta-algorithm, we now discuss its main elements.
The meta-algorithm begins by assuming that the problem is a simple
(MAB) instance and primarily uses an optimal MAB algorithm for arm
selection: this default choice is denoted by $i_t$ in Algorithm~\ref{alg:metamethod}. 
To address model selection, it uses both an
\emph{exploration schedule} and a \emph{misspecification test}, both
of which admit different instantiations. 
The exploration schedule governs a rate at which the algorithm should choose arms uniformly at random, which can be helpful for detecting misspecification. 
The misspecification test is simply a surrogate statistical test to check if the
instance is, in fact, a complex (CB) instance (i.e. $\Tstar \neq \boldsymbol{0}$).
If the test detects misspecification, we immediately switch to an optimal linear CB algorithm for the remaining time steps.

\begin{table*}[t]
% \vskip 0.15in
\begin{center}
\begin{small}
\begin{tabular}{lcc}
\hline
\rule{0pt}{12pt} Algorithm & Estimator $\Sgaphat(\cdot)$ & Forced exploration parameter $\nu_t$ \\
\hline \hline
\textsc{OSOM}~\cite{chatterji2020osom} & Plug-in estimator & $\nu_t = 0$ (no extra exploration) \\
\hline
\textsc{ModCB}~\cite{foster2019model} & Fast estimator defined in~\cite{foster2019model} & $\nu_t \asymp t^{-1/3}$\\
\hline
\textcolor{blue}{\textsc{ModCB.U}} & \textcolor{blue}{Fast estimator defined in Algorithm~\ref{alg:estimateresidualsideinfo}} & \textcolor{blue}{$\nu_t \asymp t^{-2/9}$}  \\
\hline
\textcolor{blue}{\textsc{ModCB.A}} & \textcolor{blue}{Fast estimator defined in~\cite{foster2019model}} & \textcolor{blue}{Algorithm~\ref{alg:explorationschedule}}  \\
\hline
\end{tabular}
% \end{sc}
\end{small}
\end{center}
% \vspace{-0.25cm}
\caption{Comparison of model selection algorithms in terms of their
  estimator for the square loss gap $\Sgaphat(\cdot)$ and exploration
  schedule. \label{tab:metaalgorithm}}
% \vspace{-0.1in}
\end{table*}

While~\cite{chatterji2020osom} and~\cite{foster2019model} both use the meta-algorithmic structure in Algorithm~\ref{alg:metamethod}, they instantiate it with difference choices of misspecification test and exploration schedule.
The high-level details of where the approaches diverge are summarized in Table~\ref{tab:metaalgorithm}, and the results that they obtain are summarized in Table~\ref{tab:resultsuniversal}.
We provide a brief description of the salient differences below.
% \vspace{-3mm}
\begin{enumerate}
\item \citet{chatterji2020osom} do not incorporate any forced exploration in their procedure, as evidenced by the choice of parameter $\nu_t = 0$ above for all values of $t$.
They also use the plug-in estimator of the linear model parameter $\Tstar$ to obtain an estimate of the gap in performance between the two model classes.
The error rate of this plug-in estimator scales as $\Oh(d/n)$ as a function of the number of samples $n$, and matches the putative regret bound for linear CB.
Consequently, they achieve the optimal model selection rate of Objective 1, as well as the stronger instance-optimal rate in the case of MAB, but require a strong assumption of \emph{feature diversity for each arm}; that is, they require $\Sigmabold_i \succeq \gamma \eye$ for all $i \in [K]$.
Intuitively, feature diversity
% \footnote{Such feature diversity was previously used to show that the greedy algorithm can obtain competitive performance with \textsc{LinUCB} in contextual bandits, both for regular regret minimization and fairness objectives~\citep{bastani2017mostly,kannan2018smoothed,raghavan2018externalities}.}
eliminates the need for forced exploration to successfully test for potential complex model structure.
% \vspace{-2mm}
\item \citet{foster2019model} incorporate forced exploration of an $\epsilon$-greedy-style by setting the forced exploration parameter $\nu_t = t^{-1/3}$.
This automatically precludes achieving the stronger Objective 1, but leaves the door open to achieving Objective 2 for some smaller choice of $\alpha$.
To do this, they leverage fast estimators~\cite{verzelen2018adaptive,dicker2014variance,kong2018estimating} of the gap between the two model classes, whose error rate can be verified to scale as $\Oh(\sqrt{d}/n)$ as a function of the number of samples $n$. 
This is significantly better in its dependence on $d$ than the standard plug-in estimator.
Moreover, forced exploration removes the requirement of restrictive feature diversity assumptions \emph{on each arm}; nevertheless, an \emph{arm-averaged} feature diversity assumption is still required.
Specifically, they assume that $\Sigmabold \succeq \gamma \eye$ where $\Sigmabold := \frac{1}{K} \sum_{i \in [K]} \Sigmabold_i$, which is strictly weaker than the arm-specific condition of~\citet{chatterji2020osom}.
Above, $\Sigmabold$ is the covariance matrix of the \emph{mixed} context obtained from uniform exploration, which we denote by $\x := \x_I$ where $I \sim \text{Unif}[K]$.
\end{enumerate}
%
% \vspace{-3mm}
This discussion tells us that the initial attempts at model selection~\citep{chatterji2020osom,foster2019model} fall short both in their breadth of applicability and their ability to adapt to structure in the model selection problem.
This naturally motivates the question of whether we can design new algorithms with two key properties:
% \vspace{-3mm}
\begin{itemize}
\item \emph{Universality}: Can we meet Objective 2 for some value of $\alpha \in (0,1)$ under stochastic contexts but with no additional diversity assumptions?
% \vspace{-2mm}
\item \emph{Adaptivity}: Can we meet Objective 1 under maximally favorable conditions (feature diversity for all arms), and Objective 2 otherwise?
\end{itemize}
% \vspace{-3mm}

%!TEX root = main.tex

\section{Main results}\label{sec:results}

We now introduce and analyze two new algorithms that provide a nearly complete answer to the problems of universality and adaptivity for the MAB-vs-linear CB problem.

\subsection{Universal model selection under stochasticity}\label{sec:noeigenvalue}

In this section, we present \textsc{ModCB.U}, a simple variant of \textsc{ModCB}~\citep{foster2019model} that achieves Objective 2 of model selection without requiring \emph{any} feature diversity assumptions, arm-averaged or otherwise.
Therefore, this constitutes a \emph{universal model selection} algorithm between an MAB instance and a linear CB instance.
% Consequently, we call this a universal model selection algorithm

Our starting point is the approach to model selection in~\cite{foster2019model} described above in Section~\ref{sec:priorart}.
Here, we recap the details of the fast estimator $\Sgaphat(\cdot)$ of the \emph{square-loss-gap}, which is given by $\Sgap := \E{(\x^\top \Tstar)^2} = (\Tstar)^\top \Sigmabold \Tstar$.
The square-loss-gap can be verified to be an upper bound on the expected gap of the best-in-class performance between the CB and MAB models (see~\cite{foster2019model} for details on this upper bound), but is also equal to $0$ \emph{iff} $\Tstar = \boldsymbol{0}$ (and $\Sigmabold$ is full rank).
Therefore, it is a suitable surrogate statistic to test for misspecification, as detailed in the meta-algorithmic structure of Algorithm~\ref{alg:metamethod}.
% \vm{Add sentence about $K$ factors arising from upper-bounding the true gap by square-loss gap, or not needed?}
The estimator is denoted by $\Sgaphat$, and is described as a black-box procedure in Algorithm~\ref{alg:estimateresidualsideinfo} with access to an estimator of the covariance matrix $\Sigmahat_t$ that is constructed from $t$ \emph{unlabeled} samples.
The estimator that is used by\textsc{ModCB}~\cite{foster2019model} is simply the sample covariance matrix at round $t$, defined by
\begin{align}\label{eq:samplecovariancematrix}
\Sigmahat_t := \frac{1}{Kt} \sum_{s=1}^t \sum_{i = 1}^K \x_{i,s} \x_{i,s}^\top.
\end{align}
Note that such an estimator can be easily constructed as we have access to \emph{all} past contexts
 $\{\x_{i,s}\}_{i \in [K],s \in [t]}$ at any round $t$. This effective full-information access to contexts, in fact, forms the crux of both of our algorithmic ideas.

This approach is summarized in the sub-routine Algorithm~\ref{alg:estimateresidualsideinfo}, which is instantiated in \textsc{ModCB} for any time step $t$ with 
\begin{align*}
\{\x_i,y_i\}_{i=1}^n := \{\x_{A_s,s},g_{A_s,s} - \muhat_{A_s,s}\}_{1 \leq s \leq t: U_s = 0}.
\end{align*}
That is, the set of training examples used is the set of context-reward pairs on all designated \emph{exploration rounds}.
Above, $\muhat_{i,s}$ constitutes the estimate of the sample means constructed only from past exploration rounds\footnote{This ensures the estimates are unbiased, i.e. $\E{\x y} = \Sigmabold \Tstar$.}.
As a consequence of this choice, we note that for this instantiation of the sub-routine Algorithm~\ref{alg:metamethod}, we have $n := \text{\# of exploration rounds before time step $t$}$ and $m := t$ at any given time step $t$.
% Here, the set of exploration rounds denotes the set of rounds for which $U_t = 0$, as defined in Algorithm~\ref{alg:metamethod}.
% \vm{I haven't made this as explicit as I should in the appendix: worth revising before arXiv.}}.

A key bottleneck lies in the obtainable estimation error rate of $\Sgap$: while the leading dependence is given by $\frac{\sqrt{d}}{\text{\# of exploration rounds}}$ (which is at the heart of the $\widetilde{\Oh}(d^{1/3}T^{2/3})$ rate that \textsc{ModCB} achieves), there is also an inverse dependence on the minimum eigenvalue of the arm-averaged covariance matrix $\Sigmabold$, which we denote here by $\gamma_{\min}$.
This dependence arises as a consequence of needing to estimate the inverse covariance matrix $\Omegabold := \Sigmabold^{-1}$ from unlabeled samples.
In essence, this requires $\Sigmabold$ to be \emph{well-conditioned}, in the sense that we need $\gamma_{\min}$ to be a positive constant to ensure the model selection rate of $\widetilde{\Oh}(d^{1/3}T^{2/3})$.
This precludes non-trivial model selection rates from \textsc{ModCB} for cases where $\gamma_{\min}$ could itself decay with $d$, the dimension of the contexts, or $T$, the number of rounds.
It also does not allow for cases in which $\Sigmabold$ may be singular.
% \ak{$\W(t)$ is not actually defined anymore!} \vm{Fixed, but thinking of making this explicit.}

Our first main contribution is to adjust \textsc{ModCB} to successfully achieve Objective 2 in model selection with arbitrary stochastic, sub-Gaussian contexts.
Because our algorithm achieves a \emph{universal} model selection guarantee over all stochastic context distributions, we name it \textsc{ModCB.U}.
The algorithmic procedure is identical to that of \textsc{ModCB} except for the choice of estimator for the inverse covariance matrix, $\Omegahat$, that is plugged into Algorithm~\ref{alg:estimateresidualsideinfo}.
Our key observation is as follows: if certain directions are small in magnitude for the contexts corresponding to \emph{all} arms (as will be the case when $\Sigmabold$ has vanishingly small eigenvalues), then we may not actually want try to estimate the square loss gap along them: ignoring them might be a better option.
Our approach to ignoring low-value directions simply uses eigenvalue thresholding\footnote{Note that eigenvalue thresholding is a type of \emph{hard-thresholding} approach commonly employed in high-dimensional statistics~\cite{bhatia2015robust}. 
A plausible alternative approach would be to undertake a \emph{soft-thresholding} approach, by simply adding a non-zero quantity $\gamma$ to each of the eigenvalues of the sample covariance matrix.
This is the approach taken, for example, in ridge regression, and we expect similar results to hold with soft-thresholding.
%% \vm{Commentary on ridge: feel free to edit/modify}
} to construct an improved \emph{biased} estimate of the inverse covariance matrix $\Omegahat$.
We formally define the eigenvalue thresholding operator below.

\begin{definition}\label{def:thresholdingoperator}
Define the clipping operator $[x]_v := \max\{x,v\}$. 
Then, for any matrix $\M\succeq 0$ with diagonalization $\M := \UM \boldsymbol{\Lambda_M} \UM^\top$ and any value of $\gamma > 0$, we define the thresholding operator 
\begin{align*}
    \Tgamma{\M} &:= \UM \Tgamma{\boldsymbol{\Lambda_M}} \UM^\top \quad \text{ where }  \\
    \Tgamma{\boldsymbol{\Lambda}} &:= \text{diag}([\lambda_1]_{\gamma},\ldots,[\lambda_d]_{\gamma}).
\end{align*}
\end{definition}
We use Definition~\ref{def:thresholdingoperator} to specify our (biased) estimators of the covariance and inverse-covariance matrices $\Sigmabold, \Omegabold := \Sigmabold^{-1}$.
In particular, we let $\Sigmahat_t$ denote the sample covariance matrix of $\Sigmabold$ from $t$ unlabeled samples, given in Eq. \eqref{eq:samplecovariancematrix}.
Then, our estimators are given by
\begin{align}
    \Sigmahat &:= \Tgamma{\Sigmahat_t} \quad \textrm{ and } \quad \Omegahat := \Sigmahat^{-1}, \label{eq:sigmaestimate}
\end{align}
and we simply plug the estimate $\Omegahat$ into Algorithm~\ref{alg:estimateresidualsideinfo}.
Note that $\Sigmahat$ is always invertible for any $\gamma > 0$.
In essence, this lets us set $\gamma$ as a tunable parameter to tradeoff the estimation error of a surrogate approximation to the square-loss gap $\Sgap$ (which will \emph{decrease} in $\gamma$) and the approximation error that arises from ignoring all directions with value less than $\gamma$ (which will \emph{increase} in $\gamma$).
As our first main result shows, we can set a value of $\gamma$ that scales with $d$ and $T$ and successfully achieve Objective 2 of model selection for \emph{any} stochastic sub-Gaussian context distribution.
% \ak{Mention that we simply plug this $\Omegahat$ into Algorithm~\ref{alg:estimateresidualsideinfo}.}

%% \ak{The above discussion is a litle confusing, because it's not entirely clear what is the previous approach and what is the new approach. Maybe we can spell this out more clearly?}
%% \vm{I added a few sentences clarifying the key difference between \textsc{ModCB} and \textsc{ModCB.U}. Hopefully this helps a bit.}

\begin{algorithm}[t]
	\caption{ EstimateResidual} \label{alg:estimateresidualsideinfo}
    \begin{algorithmic}
		\STATE \textbf{Input}: Examples $\{(\x_i,y_i)\}_{i=1}^n$ and second moment matrix estimate $\Sigmahat \in \R^{d \times d}$ (which can be constructed from $m \gg n$ unlabeled sample).
		\STATE Return estimator
		\begin{align*}
	    \Sgaphat := \frac{1}{\binom{n}{2}} \sum_{i < j} \Big{\langle} \Omegahat^{1/2} \x_i y_i, \Omegahat^{1/2} \x_j y_j \Big{\rangle}
		\end{align*}
        of the square-loss gap $\Sgap := (\Tstar)^\top \Sigmabold \Tstar$.
        \end{algorithmic}
	\end{algorithm}

\begin{table*}[t]
\begin{center}
\begin{small}
\begin{tabular}{|l|c|c|c|}
\hline
Algorithm & Obj. 1 (optimal rates) & Obj. 2 ($d^\alpha T^{1-\alpha}$ rates) & context assumption \\
\hline\hline
\textsc{OSOM}~\citep{chatterji2020osom} & Yes & Yes $(\alpha = 1/2$) & $\forall i\in[K]: \Sigmabold_i \succeq \gamma \eye$\\
\hline
\textsc{ModCB}~\citep{foster2019model} & No & Yes $(\alpha = 1/3)$ & $\Sigmabold \succeq \gamma \eye$ \\
\hline
\textcolor{blue}{\textsc{ModCB.U}} & \textcolor{blue}{No} & \textcolor{blue}{Yes $(\alpha = 1/6)$} & \textcolor{blue}{iid contexts only}\\
\hline
\textsc{Corral-style} & No & No & iid contexts only \\
\hline
\end{tabular}
\end{small}
% \vspace{-0.5cm}
\caption{ Comparison of model selection algorithms in terms of the
  regret guarantee and assumptions, i.e., ``universality.'' Dependence on the number of arms $K$ is
  omitted.  \label{tab:resultsuniversal}}
\end{center}
\vskip -0.1in
\end{table*}

\begin{theorem}\label{thm:universal}
\textsc{ModCB.U} with $\gamma := (d/T)^{1/3}$ achieves, with probability at least $1 - \delta$, model selection rates
\begin{align}
R_T^S = \widetilde{\Oh}_{\delta}(T^{5/6}) \quad \text{ and } \quad
R_T^C = \widetilde{\Oh}_{\delta}(d^{1/6}T^{5/6}).\label{eq:universal}
\end{align}
% with probability at least $1-\delta$. 
\end{theorem}
Equation~\eqref{eq:universal} clearly demonstrates model selection rates of the form required from Objective 2, and shows that Objective 2 can be met for some value of $\alpha$ with the sole requirement of stochasticity on the contexts.
Table~\ref{tab:resultsuniversal} allows us to compare the achievable rate to both \textsc{OSOM}~\citep{chatterji2020osom} and \textsc{ModCB}~\citep{foster2019model}; corralling approaches, which are assumption-free but meet neither Objectives 1 nor 2, are also included as a benchmark.
In particular, it is clear from a quick read of the table that as we go from \textsc{OSOM} to \textsc{ModCB} to our approach, the assumptions required on the context distributions weaken, as do the obtainable rates (recall that because the rate $d^{\alpha}T^{1-\alpha}$ decreases in $\alpha$, a guarantee with a larger value of $\alpha$ implies one with a smaller value of $\alpha$).

The proof of Theorem~\ref{thm:universal} is provided in Appendix~\ref{sec:thm1proof}.
In Appendix~\ref{sec:linearCB}, we describe how this procedure and result extends to the more complex case of linear CB under an additional assumption of block-diagonal structure on the covariance.
% , where the blocks designate the dimensions of the model classes.
% \vspace{-3mm}
\subsection{Data-adaptive algorithms for model selection}\label{sec:dataadaptive}

In this section, we introduce a new \emph{data-adaptive} exploration schedule and show that it provably achieves Objective 1 under the strongest assumption of feature diversity \emph{for each arm} (as in~\cite{chatterji2020osom}), but also achieves Objective 2 under the weaker assumption of \emph{arm-averaged} feature diversity (as in~\cite{foster2019model}).
Our key insight is that the arm-specific feature diversity condition used by~\cite{chatterji2020osom} is itself testable from past contextual information; therefore, it can be tested for \emph{before} we decide on an arm and receive a reward.

To describe this idea formally, we introduce some more notation.
At time step $t$, we denote the exploration set that we have built up thus far by $\W(t-1) \subset [t - 1]$.
Now, we use an inductive principle.
Suppose that the contexts that are present in the exploration set $\W(t-1)$ are already sufficiently ``diverse'' in a certain quantitative sense (that we will specify shortly).
% Concretely, suppose that the exploration set that we have built up thus far, which we denote by $\W(t-1)$, is sufficiently diverse.
Then, we can easily check whether the arm that we would ideally pull when the true model is simple, i.e. $i_t$ (the ``greedy'' arm), continues to preserve this property of diversity.
Importantly, because we are able to observe the contexts before making a decision, we can check this condition \emph{before} deciding on the value of $A_t$.
% If this is indeed the case, we choose not to forcibly explore, i.e. we select $A_t = i_t$ (as reflected by the case $Y_t = 1$); if not, we forcibly explore, i.e. we select $A_t \sim \text{Unif}[K]$ (as reflected by the case $Y_t = 0$).
\begin{algorithm}[t]
    \caption{Data-adaptive exploration schedule} \label{alg:explorationschedule}
    \begin{algorithmic}
        \STATE \textbf{Input}: Labeled examples $\{\x_{A_s,s},g_{A_s,s}\}_{s=1}^{t-1}$, current exploration set $\W(t-1)$, hyperparameter $\gamma > 0$.
        \STATE Construct matrices $\M_{t-1} := \sum_{s \in \W(t-1)} \x_{A_s,s} \x_{A_s,s}^\top$ and $\M_t(i_t) := \M_{t-1} + \x_{i_t,t} \x_{i_t,t}^\top$.
        \STATE Define $Z_t \sim \text{Bernoulli}(1 - t^{-1/3})$ and $Y_t = \mathbb{I}\left[\frac{1}{|\W(t)| + 1} \cdot \M_t(i_t) \succeq \gamma \eye\right]$.
        \STATE Add $t$ to ``exploration set'' $\W(t)$ if $(Y_t = 1 \text{ and } Z_t = 1)$ or $(Y_t = 0 \text{ and } Z_t = 0)$.
        \STATE \textbf{Output}: Return $A_t = i_t$ if $Y_t = 1$ \emph{or} $Z_t = 1$; $A_t \sim \text{Unif}[K]$ otherwise.
        \end{algorithmic}
    \end{algorithm}

This new sub-routine for data-adaptive exploration, which we call \textsc{ModCB.A}, is described in Algorithm~\ref{alg:explorationschedule}.
We elaborate on the algorithm description along three critical verticals: a) the decision to forcibly explore, b) the choice of estimator, and c) the designated ``exploration rounds'' that are used for the estimator.
% \vspace{-3mm}
\paragraph{When to forcibly explore:}
At time step $t$, \textsc{ModCB.A} uses the random variables $Y_t$ and $Z_t$ to decide whether to stick with the ``greedy'' arm $A_t = i_t$, or to forcibly explore, i.e. $A_t \sim \text{Unif}[K]$.
For time step $t$, the random variable $Y_t$ denotes the indicator that the diversity condition continues to be met by context $\x_{i_t,t}$.
This means that if $Y_t = 1$, we will pick $A_t = i_t$.
On the other hand, if the diversity condition is not met (i.e. $Y_t = 0$), we revert to the forced exploration schedule used by \textsc{ModCB}.
This schedule sets a variable $Z_t \sim \text{Bernoulli}(1 - t^{-1/3})$, and selects $A_t = i_t$ if $Z_t = 1$ and $A_t \sim \text{Unif}[K]$ otherwise.
% In essence, the variable $Z_t$ denotes the indicator that we decided not to forcibly explore, and $Y_t$ is an indicator that the context corresponding to ideal arm $i_t$ is well-conditioned.
% $Z_t = 1$ represents the case for which we decided not to forcibly explore anyway (regardless of the realization of $Y_t$).
% On the other hand, if $Z_t = 0$ (in which case ModCB would have forcibly explored), our data-adaptive procedure still picks $A_t = i_t$ as long as $Y_t = 1$.
In summary, we end up picking $A_t = i_t$ if $Z_t = 1$ \emph{or} $Y_t = 1$, while \textsc{ModCB} would have picked $A_t = i_t$ \emph{only if} $Z_t = 1$.
As a result, our procedure, called \textsc{ModCB.A}, allows us to adapt on-the-fly to friendly feature diversity structure (and explore much less) while preserving more general guarantees.
% \vspace{-3mm}
\paragraph{The choice of estimator $\Sgaphat$:} First, we specify the choice of estimator of the square loss gap, $\Sgaphat$ from samples in the designated exploration set $\W(t)$.
(We will specify the procedure for construction of this exploration set shortly.)
For convenience, we index the elements of the exploration set $\W(t)$ in ascending order by $s_1,\ldots,s_{|\W(t)|}$.
We also recall that $\Sigmahat_t$ denotes the sample covariance matrix
% \footnote{For technical reasons related to a \emph{random} covariate shift induced by data-adaptive exploration, eigenvalue thresholding does not work quite as well here.} 
as defined in Equation~\eqref{eq:samplecovariancematrix}.
Armed with this exploration set, we define our estimator of $\Sgaphat(\W(t))$ in accordance with the sub-routine in Algorithm~\ref{alg:estimateresidualsideinfo} with the examples from the exploration set, i.e. $\{\x_{A_{s_j},s_j},y_{A_{s_j},s_j}\}_{j=1}^{|\W(t)|}$.
In particular, we estimate an adjusted square loss gap, given by
\begin{align}
\overline{\Sgap} &:= \eu{(\Sigmahat_t)^{-1/2} \overline{\Sigmabold}_{|\W(t)|} \Tstar}_2^2, \quad \textrm{where} \nonumber \\
\overline{\Sigmabold}_t &:= \frac{1}{|\W(t)|} \cdot \E{\sum_{j=1}^{|\W(t)|} \x_{A_{s_j}} \x_{A_{s_j}}^\top}\label{eq:sigma_timevarying}
\end{align}
% \ak{$\Sigmabold_{A_{s_j}}$ not really defined?}
Note that because $\overline{\Sigmabold}_t$ is random, the adjusted square loss gap is also random; nevertheless, it turns out that it is \emph{almost surely} a good proxy for the true square loss gap $\Sgap$.
We estimate this adjusted squared loss gap with the estimator that is given by
\begin{align}
    \Sgaphat := \frac{1}{\binom{|\W(t)|}{2}} \sum_{j' < j} \Big{\langle} (\Sigmahat_t)^{-1/2} \x_{s_{j'}} y_{s_{j'}}, (\Sigmahat_t)^{-1/2} \x_{s_j} y_{s_j} \Big{\rangle}.
\end{align}
% where $\{\x_{s_j},y_{s_j}\}_{j=1}^{|\W(t)|}$ are defined just as in Section~\ref{sec:noeigenvalue}.
% \vspace{-3mm}
% \vm{Wondering whether below, we should first define the estimator with $\W(t)$ as a black-box, then specify $\W(t)$ with filtering out the samples for which $Y_t = 1,Z_t = 0$ to avoid the bias subtlety.}
% To complete our description of \textsc{ModCB.A}, we specify the surrogate test statistic $\Sgaphat$ that we use to test for complex model structure.
\paragraph{How to build the exploration set $\W(t)$:} 
To complete our description of \textsc{ModCB.A}, we specify the data-adaptive exploration set $\W(t) \subset [t]$ at round $t$ that is used for the estimation subroutine.
Notice from Algorithm~\ref{alg:explorationschedule} that we did not include the rounds for which $Y_t \neq Z_t$ in the exploration set.
Interestingly, the two cases for which this happens are undesirable for two distinct reasons, as detailed below.
% \vspace{-3mm}
\begin{itemize}
\item Rounds on which $Y_t = 0$ and $Z_t = 1$ constitute rounds on which there was no forced exploration and the context corresponding to arm $i_t$ need not be well-conditioned: therefore, we do not want to include these samples for estimation.
% \vspace{-2mm}
\item Rounds on which $Y_t = 1$ and $Z_t = 0$ are picked as a sole consequence of well-conditioning on the context $\x_{t,i_t}$.
When this condition holds, $\x_{t,i_t}$ induces good conditioning, however its distribution is affected by the filtering process, inducing bias that complicates estimating the square loss gap. 
%% The satisficing of this condition will in general induce a non-zero bias on the context $\x_{t,i_t}$ and downstream complexities in estimation.
To avoid these complexities, we filter out these rounds.
Note that there is no bias when both $Y_t=Z_t=1$, because the choice $A_t=i_t$ can be attributed to $Z_t=1$ and not because the context feature induces adequate conditioning.
(The proof of Theorem~\ref{thm:dataadaptive} highlights that these rounds would make a minimal difference to the ensuing model selection rates.)
% \vm{This is a bit clunky; feel free to reword.}
\end{itemize}
% \vm{Justification for constructing $\W(t)$ above; could be moved after estimator definition.}
% \vspace{-3mm}
This completes our description of our adaptive algorithm, \textsc{ModCB.A}.
We show below that \textsc{ModCB.A} achieves the following data-adaptive model selection guarantee.
\begin{theorem}\label{thm:dataadaptive}
\textsc{ModCB.A} with parameter choice $\gamma > 0$ achieves the following model selection rates, each with probability at least $1-\delta$:
% \vspace{-3mm}
\begin{enumerate}
    \item If feature-diversity holds \textbf{for every arm} with parameter $\loweig' \geq \loweig$, then
    \begin{align}
        R_T^S &= \widetilde{\Oh}_{\delta}\left(\sum_{i \neq i^*} \frac{\log T}{\Delta_i}\right) \text{ and } R_T^C &= \widetilde{\Oh}_{\delta}\left(\sqrt{\frac{dT}{\loweig'}}\right). \label{eq:universalmodelselection}
    \end{align}
    % \vspace{-2mm}
    \item If \textbf{arm-averaged} feature diversity is satisfied with parameter $\gamma' \geq \gamma$, then
    \begin{align}
        R_T^S &= \widetilde{\Oh}_{\delta}(T^{2/3}) \quad \text{ and } \quad
        R_T^C = \widetilde{\Oh}_{\delta}\left(\frac{1}{(\loweig')^2} d^{1/3} T^{2/3}\right). \label{eq:armaveragedmodelselection}
    \end{align}
\end{enumerate}
\end{theorem}
The proof of Theorem~\ref{thm:dataadaptive} is provided in Appendix~\ref{sec:thm2proof}.
Observe that Equation~\eqref{eq:universalmodelselection} is identical to the \textsc{OSOM} rate, and Equation~\eqref{eq:armaveragedmodelselection} is identical to the \textsc{ModCB} rate.
Consequently, our data-adaptive exploration subroutine results in a \textit{single} algorithm that achieves both rates under the requisite conditions.
As summarized in Table~\ref{tab:resultsdataadaptive}, \textsc{OSOM} will not work even under arm-averaged feature diversity if arm-specific diversity does not hold.
On the other hand, \textsc{ModCB} can be verified not to improve under the stronger condition of arm-specific feature diversity.
% \todo{Consider making this a one-column table if it helps with space.}
In conclusion, we can think of \textsc{ModCB.A} as achieving the ``best-of-both-worlds'' model selection guarantee between the two approaches, by meeting Objective 1 under arm-specific feature diversity and Objective 2 otherwise.

\begin{table}[t]
\begin{center}
\begin{small}
\begin{tabular}{|l|c|c|}
\hline
Algorithm & Arm-specific diversity & Arm-averaged diversity \\
\hline \hline
\textsc{OSOM} & $\log(T)/\textrm{gap}$ and $\sqrt{dT}$ & None  \\
\hline
\textsc{ModCB} & $T^{2/3}$ and $d^{1/3}T^{2/3}$ & $T^{2/3}$ and $d^{1/3}T^{2/3}$ \\
\hline
\textcolor{blue}{\textsc{ModCB.A}} & \textcolor{blue}{$\log(T)/\textrm{gap}$ and $\sqrt{dT}$} & \textcolor{blue}{$T^{2/3}$ and $d^{1/3}T^{2/3}$ } \\
\hline
\end{tabular}
\end{small}
% \vspace{-0.25cm}
\caption{Adaptivity properties of model selection algorithms. We list
  the regret bounds for simple and complex model under context
  diversity conditions.  By carefully adjusting the exploration
  schedule, \textsc{ModCB.A} adapts to the favorable ``arm-specific
  diversity'' setting. Dependence on the number of arms $K$ is omitted.\label{tab:resultsdataadaptive}}
\end{center}
\vskip -0.1in
\end{table}
% \vspace{-3mm}
% \vm{Something that I realized is quite annoying in the algorithm specs in main body: we have not defined test statistic ``threshold'' $\alpha_\delta(\cdot,\cdot)$ formally anywhere. (This is particularly grating in Theorem 2, where $\gamma$ is described as a hyperparameter choice and the only reason that's the case is because $\alpha$ depends on $\gamma$! All of this is in the appendix but is going to be a bit annoying for a reader. That being said the algorithm descriptions are already pretty hairy. Thoughts on bringing defs into main text?} \vm{Resolved for now}

% \subsection{Implications}\label{sec:implications}

%!TEX root = main.tex

\section{Discussion and future work}

In this paper, we introduced improved statistical estimation routines and exploration schedules to plug-and-play with model selection algorithms.
The result of these improvements is that we advance the state-of-the-art for model selection along the axes of \emph{universality} and \emph{adaptivity} (as defined at the end of Section~\ref{sec:setup}).
Our results are most complete for the model selection problem of MAB-vs-linear CB, but Appendix~\ref{sec:linearCB} presents some extensions to the problem of model selection among linear contextual bandits.
Given the recent interest and sharp results for model selection in the linear setting, it is natural to ask whether these ideas extend to any nonlinear setting.
Recent work~\cite{marinov2021pareto} constructed nonlinear function classes for which Objective 2 cannot be achieved even when the contexts are stochastic.
However, identifying nonlinear function classes for which Objective 1 or Objective 2 of model selection is possible remains open and is an important direction for future work.
We believe that achieving Objective 2 requires exploiting specific properties of function classes such as the ability to test for misspecification at a fast rate.
However, the ideas in \textsc{OSOM}~\cite{chatterji2020osom} and our data-adaptive exploration routine \textsc{ModCB.A} are more model-agnostic and instead exploit quantities that depend only on the data distribution.
Consequently, we believe they can be generalized to achieve Objective 1 for general function classes when it is possible.
To this end, Appendix~\ref{sec:generalroadmap} presents a generalization of the arm-specific diversity condition that is motivated by statistical concepts in transfer learning/covariate-shift~\cite{quinonero2008dataset}, and a consequent possible extension of the principles in \textsc{OSOM} and \textsc{ModCB.A}.
% We also discuss examples of nonlinear function classes for which \textsc{ModCB.A} could be directly applied.
Some parts of this extension are relatively straightforward, and we sketch how to do them in Appendix~\ref{sec:generalroadmap}; however, obtaining strong end-to-end guarantees requires more effort and is an interesting direction for future work.

% Our work also leaves some remaining questions open in the purview of linear contextual bandit model selection --- such as a) extending our eigenvalue-thresholding approach to general linear contextual bandits, which is challenging due to the presence of cross-correlation terms, and b) characterizing instance-optimal model selection rates.

\section*{Acknowledgements}
VM acknowledges helpful initial discussions with Weihao Kong, and support from an Adobe Data Science Research Award and a Simons-Berkeley Research Fellowship for the program “Theory of Reinforcement Learning” in Fall 2020. This work was done in part while the authors were visiting the Simons Institute for the Theory of Computing.

\newpage

\bibliography{ref}

\newpage
\appendix
\onecolumn

%!TEX root = main.tex

\section{Proof of Theorem~\ref{thm:universal}}\label{sec:thm1proof}

The following lemma (intended to replace Theorem 2 in~\citet{foster2019model}) characterizes how the estimation error of our thresholded estimator $|\Sgaphat - \Sgap|$ will depend on the choice of $\gamma$.

\begin{lemma}\label{lem:truncatedestimator}
Suppose that we have $n$ labeled samples and $m$ unlabeled samples.
Provided that $m \geq C(d + \log(2/\delta))/\gamma$, the estimator provided in Algorithm~\ref{alg:estimateresidualsideinfo} guarantees that
\begin{align}\label{eq:esterrorrate}
    |\Sgaphat - \Sgap| &\leq \frac{1}{2} \Sgap + \alpha_{\delta}(n,m) \text{ where } \\
    \alpha_{\delta}(n,m) &:= \mathcal{O}\left(\frac{1}{\gamma} \cdot \frac{d^{1/2} \log^2(2d/\delta)}{n} + \frac{1}{\gamma^4} \cdot \eu{\E{\x y}}_2^2 \cdot \frac{d + \log(2/\delta)}{m} +\gamma \right) \nonumber
\end{align}
with probability at least $1 - \delta$.
\end{lemma}
This is similar to the bound in~\citet{foster2019model} (with slightly improved inverse dependences on the threshold $\gamma$ due to the relative simplicity of the MAB-vs-linear CB setting), except that we are not assuming any spectral conditions on $\Sigmabold$ and we incur an extra additive term of $\Oh(\gamma)$ in the estimation error arising from the bias induced by the thresholding operator. 
For comparison, the bound provided in~\citet{foster2019model} is for the choice
\begin{align}\label{eq:alphamodcb}
\alpha_{\delta}(n,m) &:= \mathcal{O}\left(\frac{1}{\gamma^2} \cdot \frac{d^{1/2} \log^2(2d/\delta)}{n} + \frac{1}{\gamma^4} \cdot \eu{\E{\x y}}_2^2 \cdot \frac{d + \log(2/\delta)}{m}\right),
\end{align}
but only holds if we have $\Sigmabold \succeq \gamma \eye$.

Before proving Lemma~\ref{lem:truncatedestimator}, we sketch how it leads to the statement provided in Theorem~\ref{thm:universal}.
We follow the outline that is given in Appendix C.2.4 of~\citet{foster2019model}.
An examination of that proof, specialized to the case of $2$ model classes (in our case, MAB and linear CB), demonstrates that the dominant terms in the overall regret under the complex model (see, e.g. Eqs. (19), (20), (21) and (22) in Appendix C.2.4 of~\cite{foster2019model}) are given by
\begin{align*}
T \sqrt{K \cdot \alpha_{\delta}(|\W(T)|,T)} + |\W(T)|,
\end{align*}
where $\W(T)$ denotes the set of designated exploration rounds.
We set $\nu_t = t^{-\kappa}$ to be the forced exploration parameter (as defined in Algorithm~\ref{alg:metamethod}), and specify a choice of $\kappa$ subsequently.
Just as in~\cite{foster2019model}, we then have $|\W(T)| \leq \sqrt{\log (2/\delta)} K^{\kappa} T^{1 - \kappa}$ with probability at least $1 - \delta$.
Plugging $n := |\W(T)|$ and $m := T$ into Lemma~\ref{lem:truncatedestimator} then gives us
\begin{align*}
R_T &= \Oh\left(T\sqrt{K \cdot \left(\frac{1}{\gamma} \cdot \frac{d^{1/2} \log^2(2d/\delta)}{K^{\kappa} T^{1 - \kappa}} + \frac{1}{\gamma^4} \cdot \eu{\E{\x y}}_2^2 \cdot \frac{d + \log(2/\delta)}{T} +\gamma\right)} + \sqrt{\log (2/\delta)} K^{\kappa} T^{1 - \kappa}\right) \\
&= \Oh_{\delta}\left(K^{\kappa} T^{1-\kappa} + \frac{1}{\sqrt{\gamma}} \cdot K^{\frac{1}{2}(1 - \kappa)} \cdot d^{1/4} \cdot T^{\frac{1 + \kappa}{2}} + \frac{1}{\gamma^2} \cdot \sqrt{KdT} + \gamma T \sqrt{K}\right)
\end{align*}
with probability at least $1 - \delta$.
% \vm{Whoa, I had been wasteful earlier! Dependences on $\gamma$ go into square roots and improve things (as $1/\gamma > 1$ always).}
Note that the extra $\gamma T\sqrt{K}$ term comes from the estimation error due to misspecification (bias) that we now incur.
We now need to select the truncation amount $\gamma$ and the exploration factor $\kappa$ to minimize the above expression.
One way of doing this is given by equating the third and fourth terms (ignoring universal constants and log factors).
This gives us $\gamma^3 = \sqrt{\frac{d}{T}}$.
Substituting this into the above gives us
\begin{align*}
R_T &= \Oh_{\delta}\left(K^{\kappa} T^{1-\kappa} + \left(\frac{T}{d}\right)^{1/12} K^{\frac{1}{2}(1 - \kappa)} T^{\frac{1}{2}(1 + \kappa)} \cdot d^{1/4} + \sqrt{K} d^{1/6} T^{5/6}\right),
\end{align*}
and further substituting $\kappa = 2/9$ gives us
\begin{align*}
R_T &= \Oh_{\delta}\left(K^{2/9} T^{7/9} + K^{7/18} \cdot d^{1/6} \cdot T^{7/9} + \sqrt{K} d^{1/6} T^{5/6}\right) = \widetilde{\Oh}(d^{1/6}T^{5/6}),
\end{align*}
which clearly satisfies the form $R_T = \widetilde{\Oh}(d^{\alpha}T^{1-\alpha})$ for the case $\alpha = 1/6$.

Now that we understand how Lemma~\ref{lem:truncatedestimator} leads to Theorem~\ref{lem:truncatedestimator}, let us prove it.
\begin{proof}
Before beginning the proof, we define a term called the \textit{truncated square-loss gap} as below:
\begin{align}\label{eq:truncatedsquarelosslinearCB}
   \Sgaptilde := (\Sigmabold \Tstar)^\top \Tgamma{\Sigmabold}^{-1} \Sigmabold \Tstar.
\end{align}
We also recall that we defined $\Omegabold := \Sigmabold^{-1}$ and $\Omegahat := \Sigmahat^{-1}$, where recall that $\Sigmahat$ is the truncated second moment estimate. The proof is carried out in three distinct steps:
\begin{enumerate}
    \item Upper-bounding $\left|\Sgaphat - \E{\Sgaphat}\right|$, the ``variance" estimation error arising from $n$ samples.
    \item Upper-bounding $\left|\E{\Sgaphat} - \Sgaptilde\right|$, the bias-term with respect to the truncated squared loss gap.
    \item Upper-bounding $\left|\Sgaptilde - \Sgap\right|$, the bias arising from truncation.
\end{enumerate}

\textit{1. Upper-bounding $\left|\Sgaphat - \E{\Sgaphat}\right|$.}
We note that $\E{\Sgaphat} = \eu{\Sigmahat^{1/2}\Omegahat \E{\x y}}_2^2$.
We consider the random vector
\begin{align*}
    \Sigmahat^{1/2}\Omegahat \x y - \Sigmahat^{1/2} \Omegahat \E{\x y},
\end{align*}
and show that it is sub-exponential with parameter $\mathcal{O}(2/ \gamma)$.
This follows because $\eu{\Sigmahat^{1/2} \Omegahat}_{\op} \leq \eu{\Omegahat^{1/2}}_{\op} \leq \frac{1}{\sqrt{\gamma}}$, where the second-last inequality follows by the definition of the truncation operator. 
% \ak{Why is there $1 + ...$ here? Can't we directly get $\eu{\Sigmahat^{1/2}\Omegahat}_{\op} = \eu{\Omegahat^{1/2}}_{\op} \leq 1/\sqrt{\gamma}$?}
% \vm{Yes, corrected. I think at some point this was a carryover from more general linear case, and was loose.}

Thus, using the sub-exponential tail bound just as in Lemma 17,~\cite{foster2019model}, we get 
\begin{align*}
    \left|\Sgaphat - \E{\Sgaphat}\right| &= \mathcal{O}\left(\frac{1}{\gamma} \cdot \frac{d^{1/2} \log^2(2d/\delta)}{n} + \frac{1}{\sqrt{\gamma}} \cdot \frac{\eu{\Sigmahat^{1/2}\Omegahat \E{\x y}}_{2} \log(2/\delta)}{\sqrt{n}}\right) .
\end{align*}
Now, we note that $\eu{\Sigmahat^{1/2} \Omegahat \E{\x y}}_{2} = \sqrt{\E{\Sgaphat}}$.
Therefore, we apply the AM-GM inequality to deduce that
\begin{align}\label{eq:prooffirstpart}
    \left|\Sgaphat - \E{\Sgaphat}\right| &\leq \frac{1}{8} \E{\Sgaphat} +  \mathcal{O}\left(\frac{1}{\gamma} \cdot \frac{d^{1/2} \log^2(2d/\delta)}{n} \right)
\end{align}

\textit{2. Upper-bounding $\left|\E{\Sgaphat} - \Sgaptilde\right|$.} We denote $\mubold := \E{\x y}$ as shorthand.
It is then easy to verify that $\E{\Sgaphat} = \lin{\Omegahat \mubold,\mubold}$ and $\Sgaptilde = \lin{\Tgamma{\Sigmabold}^{-1} \mubold, \mubold}$.
Then, following an identical sequence of steps to~\cite{foster2019model}, we get
\begin{align*}
\left|\E{\Sgaphat} - \Sgaptilde\right| \leq \frac{1}{8} \Sgaptilde + \mathcal{O}\left(\eu{(\Omegahat - \Tgamma{\Sigmabold}^{-1}) \mubold}_2^2\right)
\end{align*}
We now state and prove the following lemma on operator norm control.

\begin{lemma}\label{lem:opnormcontrol}
% We have
% \begin{align}\label{eq:opnormcontrol}
%     \eu{\Tgamma{\Sigmabold}^{-1/2} \Sigmahat \Tgamma{\Sigmabold}^{-1/2} - \eye}_{\op} = \epsilon := \Oh\left(\frac{1}{\gamma} \cdot \sqrt{\frac{d + \log(2/\delta)}{m}}\right) 
% \end{align}
% %
% with probability at least $1 - \delta$.
We have
\begin{align}\label{eq:opnormequations}
    \eu{\Omegahat - \Tgamma{\Sigmabold}^{-1}}_{\op} &\leq \Oh\left(\frac{\epsilon}{\gamma^2} \right),
\end{align}
where we denote $\epsilon := \sqrt{\frac{d + \log(2/\delta)}{m}}$ as shorthand.
% \ak{Not sure why the first implies the second directly. I think if the first one were instead $\eu{\Tgamma{\Sigmabold}^{1/2} \Sigmahat^{-1} \Tgamma{\Sigmabold}^{1/2} - \eye}_{\op}$ then it would imply the other one.}
% \vm{Yes, this was a wasteful way to write it - if you look at the proof we just end up proving Equation~\eqref{eq:opnormequations} directly. I've changed statement above, and cleaned up proof.}
\end{lemma}
Note that substituting Lemma~\ref{lem:opnormcontrol} above directly gives us 
\begin{align}\label{eq:proofsecondpart}
    \left|\E{\Sgaphat} - \Sgaptilde\right| \leq \frac{1}{8} \Sgaptilde + \mathcal{O}\left(\frac{\eu{\mubold}_2^2 \cdot (d + \log(2/\delta))}{\gamma^4 m}\right) .
\end{align}
We will prove Lemma~\ref{lem:opnormcontrol} at the end of this proof.

\textit{3. Upper-bounding $\left|\Sgaptilde - \Sgap\right|$.}
Observe that 
\begin{align*}
    \Sgaptilde &= (\Tstar)^\top \Sigmabold \Tgamma{\Sigmabold}^{-1} \Sigmabold \Tstar \quad \text{ and } \quad 
    \Sgap =  (\Tstar)^\top \Sigmabold \Tstar .
\end{align*}
This directly implies that
\begin{align*}
    |\Sgaptilde - \Sgap| &= |(\Tstar)^\top (\Sigmabold \Tgamma{\Sigmabold}^{-1} \Sigmabold - \Sigmabold) \Tstar|
    \leq \eu{(\Sigmabold \Tgamma{\Sigmabold}^{-1} \Sigmabold - \Sigmabold)}_{\op},
\end{align*}
where the second inequality follows because we have assumed bounded signal, i.e. $\eu{\Tstar}_2 \leq 1$.
It remains to control the operator norm terms above.
We denote $\Sigmabold := \Ubold \Lambdabold \Ubold^\top$, and note that $\Tgamma{\Sigmabold}^{-1} := \Ubold \Tgamma{\Lambdabold}^{-1} \Ubold^\top$.
Thus, we get
\begin{align*}
    \eu{(\Sigmabold \Tgamma{\Sigmabold}^{-1} \Sigmabold - \Sigmabold)}_{\op} &= \eu{(\Lambdabold \Tgamma{\Lambdabold}^{-1} \Lambdabold - \Lambdabold)}_{\op} \leq \gamma .
\end{align*}
Putting these together gives us
\begin{align}\label{eq:proofthirdpart}
    \left|\Sgaptilde - \Sgap\right| \leq \gamma .
\end{align}
Thus, putting together Equations~\eqref{eq:prooffirstpart},~\eqref{eq:proofsecondpart} and~\eqref{eq:proofthirdpart} completes the proof.
It only remains to prove Lemma~\ref{lem:opnormcontrol}, which we now do.
\end{proof}
\begin{proof}[Proof of Lemma~\ref{lem:opnormcontrol}]
% First, we note that Equation~\eqref{eq:opnormcontrol} directly implies Equation~\eqref{eq:opnormequations} \ak{Why? I don't see this actually...}.\vm{Removed; this was not necessary.}
First, recall that $\Sigmahat := \Tgamma{\Sigmahat_m}$, and so we really want to upper bound the quantity $\eu{\Tgamma{\Sigmahat_m} - \Tgamma{\Sigmabold}}_{\op}$.
It is well known (see, e.g.~\cite{boyd2004convex}) that for any positive semidefinite matrix $\boldsymbol{M}$, the operator $\Tgamma{\boldsymbol{M  }} - \gamma \eye $ is a proximal operator with respect to the convex nuclear norm functional.
The non-expansiveness of proximal operators then gives us
\begin{align*}
    \eu{\Tgamma{\Sigmahat_m}- \Tgamma{\Sigmabold}}_{\op} 
    = \eu{\Tgamma{\Sigmahat_m} - \gamma \eye - ( \Tgamma{\Sigmabold} - \gamma \eye)}_{\op} 
    \leq \eu{\Sigmahat_m - \Sigmabold}_{\op} = \Oh(\epsilon)
\end{align*}
with probability at least $1 - \delta$.
Here, the last step follows by standard arguments on the concentration of the empirical covariance matrix. 
Recall that we defined $\epsilon := \sqrt{\frac{d + \log(2/\delta)}{m}}$ as shorthand.
% \ak{Make sure to define $\epsilon$?}

We will now use this to show that Equation~\eqref{eq:opnormequations} holds.
We have $\Tgamma{\Sigmabold}^{-1} - \Omegahat = \Tgamma{\Sigmabold}^{-1} - \Sigmahat^{-1} = (\Tgamma{\Sigmabold})^{-1}(\Sigmahat - \Tgamma{\Sigmabold}) \Sigmahat^{-1}$.
By the sub-multiplicative property of the operator norm, we then get
\begin{align*}
    \eu{(\Tgamma{\Sigmabold})^{-1} - \Sigmahat^{-1}}_{\op} &\leq \eu{(\Tgamma{\Sigmabold})^{-1}}_{\op} \eu{\Sigmahat - \Tgamma{\Sigmabold}}_{\op} \eu{\Sigmahat^{-1}}_{\op}
    \leq \frac{1}{\gamma^2} \eu{\Sigmahat - \Tgamma{\Sigmabold}}_{\op} = \Oh\left(\frac{\epsilon}{\gamma^2}\right)
\end{align*}
where the second-to-last inequality is a consequence of the definition of the truncation operation.
This shows Equation~\eqref{eq:opnormequations}, and completes the proof.
% Similarly, Equation~\eqref{eq:opnormcontrol} follows because we have
% \begin{align*}
%     \eu{\Tgamma{\Sigmabold}^{-1/2} \Sigmahat \Tgamma{\Sigmabold}^{-1/2} - \eye}_{\op} \leq \eu{\Tgamma{\Sigmabold}}_{\op} \cdot \eu{\Sigmahat - \Tgamma{\Sigmabold}}_{\op} = \frac{1}{\gamma} \cdot \gamma \epsilon = \epsilon .
% \end{align*}
%
% This completes the proof. 
% \ak{Make sure to define $\epsilon$ and maybe we should explain how this choice arises. I guess $\epsilon$ has a $1/\gamma$ factor right?}
% \vm{Should I remove Equation~\eqref{eq:opnormcontrol} completely? In this simpler MAB-vs-CB setting, it serves essentially no purpose.}
\end{proof}
% With the proof of Lemma~\ref{lem:opnormcontrol} complete, we have completed the proof of Lemma~\ref{lem:truncatedestimator}.

\section{Proof of Theorem~\ref{thm:dataadaptive}}\label{sec:thm2proof}

\begin{proof}
Our proof constitutes a deterministic proof working on various events used in~\citet{chatterji2020osom} and~\citet{foster2019model} as well as additional high-probability events that we will define.
Recall that for each value of $t = 1,\ldots,T$, we defined the filtration $\H_t := \{A_s,G_{A_s,s}\}_{s=1}^t$.

\paragraph{Meta-analysis:}

We begin the analysis by providing a common lemma for both cases that will characterize a high-probability regret bound as a functional of two random quantities: a) $|\W(t)|$, the number of designated exploration rounds \emph{that we use for fast estimation}, and $|\T(t)|$, the total number of \emph{forced-exploration} rounds.
Here, we define
\begin{align*}
\T(t) := \{s \in [t]: Z_s = 0 \text{ and } Y_s = 0\}.
\end{align*}
It is easy to verify that by definition, we have $\T(t) \subset \W(t)$. 
Indeed, recall from the pseudocode in Algorithm~\ref{alg:explorationschedule} that we defined
\begin{align*}
\W(t) := \{s \in [t]: (Z_s = 0 \text{ and } Y_s = 0) \text{ or } (Z_s = 1 \text{ and } Y_s = 1)\}.
\end{align*}
We first state our guarantee on estimation error. For any $1 \leq s
\leq t$, we define
\begin{align}\label{eq:alphadataadaptive}
\alpha_{\delta}(s,t) :=
\Oh\left(\frac{1}{\gamma} \cdot \frac{d^{1/2}
  \log^2(2d/\delta)}{s} + \frac{1}{\gamma^4} \cdot \frac{d +
  \log(2/\delta)}{t}\right).
\end{align}

\begin{lemma}\label{lem:esterrordataadaptive}
For every $t \geq 1$, we have
\begin{align}\label{eq:esterrordataadaptive}
    |\Sgaphat(\W(t)) - \Sgaptilde| &\leq \frac{1}{2} \Sgaptilde + \alpha_\delta(|\W(t)|, t)
\end{align}
with probability at least $1 - \delta$, and $\Sgaptilde$ is the adjusted square loss gap given by
\begin{align*}
    \Sgaptilde := \eu{\Sigmabold^{-1/2} \Sigmabar_{t} \Tstar}_2^2,
\end{align*}
and $\Sigmabar_{t}$ was defined in Equation~\eqref{eq:sigma_timevarying}.
\end{lemma}

\begin{proof}
This proof essentially constitutes a martingale adaptation of the proof of fast estimation in~\cite{foster2019model}.
Let $\tau(n)$ denote the \emph{random} stopping time at which $n$ exploration samples have been collected.
Moreover, let $s_1,\ldots,s_n$ denote the (again random) times at which exploration samples were collected, and $A_{s_1},\ldots,A_{s_n}$ denote the corresponding actions that were taken.
Then, we define a time-averaged covariance matrix as
\begin{align*}
\Sigmabar_n := \frac{1}{n} \cdot \E{\sum_{j=1}^n \x_{A_{s_j}} \x_{A_{s_j}}^\top}
\end{align*}
for every value of $n \geq 1$.
We state the following technical lemma, which is proved in Appendix~\ref{sec:technical} and critically uses the fact that the rounds on which $A_t = i_t$ is picked as a sole consequence of well-conditioning of the context $\x_{i_t,t}$ (i.e. if $Y_t = 1$ and $Z_t = 0$) are filtered out of the considered exploration set $\W(t)$.
\begin{lemma}\label{lem:sigmabar_t}
Assume that $\Sigmabold_i \preceq \eye$ for all $i \in [K]$.
Then, we have $\gamma \eye \preceq \Sigmabar_n \preceq \eye$ for all values of $n \geq 1$.
\end{lemma}
We will use Lemma~\ref{lem:sigmabar_t} to prove Lemma~\ref{lem:esterrordataadaptive}.
First, we recall the definition of the adjusted square loss gap,
\begin{align*}
    \Sgapbar(\W(t)) &:= \eu{\Sigmahat_t^{-1/2} \Sigmabar_t \Tstar}_2^2
\end{align*}
where we overload notation and define $\Sigmabar_t := \Sigmabar_{|\W(t)|}$ as shorthand.
Note that $\Sigmabar_t$ is a random quantity because $|\W(t)|$ is random.
By Lemma~\ref{lem:sigmabar_t}, we have that $\gamma \eye \succeq \Sigmabar_t \preceq \eye$ \emph{almost surely}. 
% \ak{I think we cannot use $\Sigmabar_{|\W(t)|}$ since the value formally depends on the actual rounds in $\W(t)$ and not just the number of rounds?}
% \vm{Hmm, I am thinking of this in ``reward tape'' sense, i.e. the sample mean depends on the actual rounds at which an arm was sampled and not just the number of times the arm was sampled. I guess the issue is with notation though...we could make this $\Sigmabar_t$ if it is easier (e.g. sample means are typically denoted by $\muhat_{i,t}$). But let's discuss.}

Similar to the proof of Lemma~\ref{lem:truncatedestimator}, the analysis proceeds in two parts:

\noindent \textit{1. Upper-bounding $\left|\Sgaphat(\W(t)) - \Sgapbar(\W(t))\right|$:}
For a given index $1 \leq j \leq |\W(t)|$, we consider the random vector
\begin{align*}
    \xibold_{t,j} := \Sigmahat_t^{-1/2}\left(\x_{s_j,A_{s_j}} y_{A_{s_j}} - \E{\x_{s_j,A_{s_j}} y_{A_{s_j}} |\H_{s_j - 1} }\right).
\end{align*}
Now, by an identical argument to that provided in~\cite{foster2019model}, we have that for each $j \in [|\W(t)|]$, the random vector $\xibold_{t,j}$ is conditionally sub-exponential with parameter $\mathcal{O}(2/ \gamma)$.
Consequently, using a martingale version of the sub-exponential tail bound (see, e.g., Chapter 2,~\cite{wainwright2019high}), we get 
\begin{align*}
    \left|\Sgaphat(\W(t)) - \Sgapbar(\W(t))\right| &= \mathcal{O}\left(\frac{1}{\gamma} \cdot \frac{d^{1/2} \log^2(2d/\delta)}{|\W(t)|} + \frac{1}{\sqrt{\gamma}} \cdot \frac{\eu{\Sigmahat_t^{-1/2} \Sigmabar_t \Tstar}_{2} \log(2/\delta)}{\sqrt{|\W(t)|}}\right) .
\end{align*}
with probability at least $1 - \delta$ for all $t \geq 1$.
% Now, we again use Lemma~\ref{lem:sigmabar_t} (the lower bound on $\Sigmabar_t$) to get $\sqrt{\gamma} \cdot \eu{\Sigmahat_t^{-1/2} \Tstar}_{2} \leq \sqrt{\Sgapbar(\W(t))}$.
Therefore, we apply the AM-GM inequality to deduce that
\begin{align}\label{eq:prooffirstpartada}
    \left|\Sgaphat(\W(t)) - \Sgapbar(\W(t))\right| &\leq \frac{1}{8} \Sgapbar(\W(t)) +  \mathcal{O}\left(\frac{1}{\gamma} \cdot \frac{d^{1/2} \log^2(2d/\delta)}{|\W(t)|} \right).
\end{align}

\noindent \textit{2. Upper-bounding $\left|\Sgapbar(\W(t)) - \Sgaptilde\right|$.}
We note that $\Sgaptilde = \eu{\Sigmabold^{-1/2} \Sigmabar_{t} \Tstar}_2^2$, and $\Sgapbar(\W(t)) := \eu{\Sigmahat_t^{-1/2} \Sigmabar_{t} \Tstar}_2^2$, where recall that $\Sigmabold$ denotes the action-averaged covariance matrix.
Further, recall that $\Sigmahat_t$ is the sample covariance matrix constructed from $t$ \textit{unlabeled} samples.
Thus, applying Lemma~\ref{lem:opnormcontrol}, following an identical sequence of steps to~\citet{foster2019model}, and using that $\Sigmabar_{t} \preceq \eye$ almost surely and $\eu{\Tstar}_2 \leq 1$, we get
\begin{align}\label{eq:proofsecondpartada}
    \left|\Sgapbar(\W(t)) - \Sgaptilde\right| \leq \frac{1}{8} \Sgaptilde + \Oh\left(\frac{1}{\gamma^4} \cdot \frac{d + \log(2/\delta)}{t}\right).
\end{align}
Now, putting Equations~\eqref{eq:prooffirstpartada} and~\eqref{eq:proofsecondpartada} together, we get
\begin{align*}
    |\Sgaphat(\W(t)) - \Sgaptilde| \leq \frac{1}{2} \Sgaptilde + \mathcal{O}\left(\frac{1}{\gamma} \cdot \frac{d^{1/2} \log^2(2d/\delta)}{|\W(t)|} + \frac{1}{\gamma^4} \cdot \frac{d + \log(2/\delta)}{t} \right).
\end{align*}
This completes the proof.
\end{proof}

Next, we prove the following meta-lemma that characterizes the simple model (SM) and complex model (CM) regret purely in terms of the size of the designated exploration set $\W(T)$ and the forced exploration set $\T(T)$ with high probability.
\begin{lemma}\label{lem:metaanalysis}
Let $\tauswitch$ denote the last round before which the algorithm switches to the complex model, if any (otherwise, we define $\tauswitch := T$).
Then, for any $\gamma > 0$ and $\delta > 0$, the following result holds with probability at least $1 - \delta$ for model selection between MAB and CB:
\begin{align*}
    R_T^S &= \Oh\left(\sum_{i \neq i^*} \frac{1}{\Delta_i} \log \left(\frac{2KT}{\Delta_i \delta}\right) + |\T(T)| \right) \text{ and } \\
    R_T^C &= \Oh\left(|\T(T)| + \sqrt{d T}\left(1 + \frac{1}{\gamma^2}\right) + \frac{d^{1/4}}{\gamma^2} \cdot \frac{\tauswitch}{\sqrt{|\W(\tauswitch)|}}\right)
\end{align*}
\end{lemma}

\begin{proof}
% For this proof, we define $\alpha_{\delta}(\W(t),t)$ as a high-probability upper bound on the estimation error $\Sgaptilde - \Sgaphat(\W(t))$, where $\Sgaptilde$ was defined in Lemma~\ref{lem:esterrordataadaptive}.
% \ak{Can we just use the definition of $\alpha_{\delta}$ from previously?}
For this proof, we work on the event
\begin{align*}
    \Acal_0 := \left\{\left|\Sgaptilde - \Sgaphat(\W(t))\right| \leq \frac{1}{2} \Sgaptilde + \alpha_{\delta}(|\W(t)|,t) \text{ for all } t = 1,\ldots,T\right\},
\end{align*}
where $\alpha_{\delta}(\cdot,\cdot)$ was defined in Equation~\eqref{eq:alphadataadaptive}.
Lemma~\ref{lem:esterrordataadaptive} showed that this happens with probability at least $1 - \delta$.
Further, we denote $g_{\max} = \E{\max_{i \in [K]} \x_i^\top \Tstar + \mu_i}$ where the expectation is over the contexts $\{\x_i\}_{i \in [K]}$ drawn identically to the contexts $\{\x_{i,t}\}_{i \in [K]}$ for any round $t \geq 1$.
We now have two cases to analyze:
\begin{enumerate}
    \item The case where the true model is SM.
    In this case, note that $\Sgaptilde = 0$ by definition and the event $\Acal_0$ directly gives us, for all $t \geq 1$,
    \begin{align*}
    \Sgaphat(\W(t)) \leq 0 + 0 + \alpha_{\delta}(|\W(t)|,t)
    \end{align*}
    with probability at least $1 - \delta$, and so the condition for elimination is never met.
    Since under this event, we stay in the simple model, we get pseudo-regret
    \begin{align*}
        R_T^S &\leq \sum_{t \notin \T(T)} (\mu^* - \mu_{A_t}) + |\T(T)| 
        = \sum_{t \notin \T(T)} (\mu^* - \mu_{i_t}) + |\T(T)| \\
        &= \Oh\left(\sum_{i \neq i^*} \frac{1}{\Delta_i} \log \left(\frac{2KT}{\Delta_i \delta}\right) + |\T(T)|\right).
    \end{align*}
    Above, the first step uses that rewards are bounded. The second step uses that by the definition of $\T(T)$ and the fact that we never switch, we will have $A_t = i_t$ for all $t \notin \T(T)$. The third step uses the fact that SM updates are only made using the rounds that are not ``forced-exploration", i.e. not in $\T(T)$. 
    % \vm{TODO - make this explicit in algorithm definition; it's useful for clean analysis.}
    This completes the proof of the lemma for the case of SM.
    \item The case where the true model is CM.
    Let $\tauswitch$ denote the last round before which the algorithm switches to CM, if any (otherwise, we define $\tauswitch := T$).
    It suffices to bound the regret until time $\tauswitch$ (when we will be playing the SM).
    First, we note that because because we have not yet switched, we have $\Sgaphat(\W(\tauswitch)) \leq \alpha_{\delta}(|\W(\tauswitch)|,\tauswitch)$.
    Therefore, we get
    \begin{align*}
        \Sgaptilde \leq 2 \Sgaphat(\W(\tauswitch)) + \alpha_{\delta}(|\W(\tauswitch)|,\tauswitch) \leq 3 \alpha_{\delta}(|\W(\tauswitch)|,\tauswitch).
    \end{align*}
    Furthermore, we note that $\Sgaptilde \geq \eu{\Tstar}_2^2 \cdot \gamma^2$ (because we have $\Sigmabar_{\tauswitch} \succeq \gamma \eye$ almost surely, and $\Sigmabold^{-1/2} \succeq I$).
    This gives us 
    \begin{align}\label{eq:snrbound}
        \eu{\Tstar}_2^2 \leq \frac{3}{\gamma^2} \cdot  \alpha_{\delta}(|\W(\tauswitch)|,\tauswitch).
    \end{align}
    Next, we get
    \begin{align*}
        R_T^C &= \sum_{t=1}^{\tauswitch} (\x_{\kappa_t,t}^\top \Tstar + \mu_{\kappa_t} - \x_{A_t,t}^\top \Tstar - \mu_{A_t}) \\
        &\leq \sum_{t \notin \T(\tauswitch)} (\x_{\kappa_t,t}^\top \Tstar + \mu_{\kappa_t} - \x_{i_t,t}^\top \Tstar - \mu_{i_t}) + |\T(\tauswitch)| \\
        &\leq \tauswitch \cdot \eu{\Tstar}_2 + \sum_{t \notin \T(\tauswitch)} \mu^* - \mu_{i_t} + |\T(\tauswitch)|\\
        &\leq \frac{T}{\gamma} \cdot \sqrt{\alpha_{\delta}(\W(\tauswitch),\tauswitch)} + \sum_{i \neq i^*} \frac{1}{\Delta_i} \log \left(\frac{2KT}{\Delta_i \delta}\right) + |\T(\tauswitch)| \\
        &= \Oh\left(|\T(\tauswitch)| + \sqrt{d\tauswitch}\cdot \frac{1}{\gamma^3} + \frac{d^{1/4}}{\gamma^2} \cdot \frac{\tauswitch}{\sqrt{|\W(\tauswitch)|}}\right) \\
        &= \Oh\left(|\T(T)| + \sqrt{dT}\cdot \frac{1}{\gamma^3} + \frac{d^{1/4}}{\gamma^2} \cdot \frac{\tauswitch}{\sqrt{|\W(\tauswitch)|}}\right) 
    \end{align*}
    This completes the proof for the case of CM.
    The first step is the definition of regret. The second step uses the definition of $\T(T)$ and the fact that we have not switched yet. The third step uses a sub-Gaussian tail bound (and might incur $\log K$ factors). The fourth step uses Equation~\eqref{eq:snrbound}. The fifth step substitutes the definition of $\alpha_{\delta}(\cdot,\cdot)$ from Equation~\eqref{eq:alphadataadaptive}.
\end{enumerate}

\end{proof}

Armed with this meta-lemma, we now complete the proof of Theorem~\ref{thm:dataadaptive} for the cases under which feature holds for all actions, and action-averaged feature diversity holds respectively.

\paragraph{Case 1: Universal feature diversity holds}

We need to show that $Y_t = 1$ for all $t \geq 1$ with high probability; if this is the case, the proof is a direct consequence of the techniques that are provided in~\citet{chatterji2020osom}.
% As in~\citet{chatterji2020osom}, we need to show that $Y_t = 1$ for all $t \geq 1$ w.h.p. in order to get the desired regret bound.
% Note that this directly ensures that $A_t = i_t$ for all $t \geq 1$.
We consider the following anytime statistical event.
\begin{align*}
    \Acal_1 &:= \left\{Y_t = 1 \text{ for all }  t \in \{\tau_{\min}(\delta,T),\ldots,T\}\right\}, \text{ where } 
    \tau_{\min}(\delta,T) := \left(\frac{16}{\loweig^2} + \frac{8}{3 \loweig} \right) \log \left(\frac{2dT}{\delta}\right)
\end{align*}
We now show that the event $\Acal_1$ occurs with probability at least $1 - \delta$.
Since $i_t$ is a deterministic functional of the history $\H_{t-1}$, we can directly apply Lemma 7,~\cite{chatterji2020osom} (which is itself an application of the matrix Freedman inequality) to get
\begin{align*}
    \gamma_{\min}(\M_t(i_t)) \geq 1 + \frac{\loweig t}{2} \text{ for all } t \geq \tau_{\min}(\delta,T) 
\end{align*}
with probability at least $1 - \delta$.
This clearly ensures that $\frac{1}{t} \cdot \gamma_{\min}(\M_t(i_t)) \geq \frac{\gamma}{2}$ for all $t \in \{\tau_{\min}(\delta,T),\ldots,T\}$, which is the required condition.

Connecting this to the meta-analysis above, event $\Acal_1$ ensures that $|\T(T)| = 0$ by the definition of $\T(t)$.
This completes the proof for the case where the true model is SM.

For the case of CM, we also need to \textit{lower bound} $|\W(\tauswitch)|$.
We denote $\tau_{\min} := \tau_{\min}(\delta,T)$ as shorthand for this portion of the proof.
It suffices to consider the case where $\tauswitch \geq 2 \tau_{\min}$ (as otherwise, we can simply bound $R_T^C \leq 2 \tau_{\min}$).
Note that $|\W(\tauswitch)|$ is lower bounded by the number of rounds $1 \leq t \leq \tauswitch$ for which $Y_t = 1$ and $Z_t = 1$.
Since $Z_t \sim \text{Bernoulli}(1 - \nu_t)$ and independent of $\{Y_t\}$, on the event $\Acal_1$ we have
\begin{align*}
    |\W(\tauswitch)| &\geq \sum_{t=\tau_{\min}}^{\tauswitch} Z_t.
\end{align*}
Because $\tau_{\min} \geq 8$, it is easy to verify that $1 - \nu_t \geq 0.5$ for all $t \geq \tau_{\min}$.
Then, we note that $\E{\sum_{t=\tau_{\min}}^{\tauswitch} Z_t} \geq 0.5 (\tauswitch - \tau_{\min})$ since $\nu_t \leq 1/2$ for all $t \geq \tau_{\min}$.
Applying Hoeffding's inequality then gives us $\sum_{t=\tau_{\min}}^{\tauswitch} Z_t \geq 0.5 (\tauswitch - \tau_{\min}) - \Oh\left(\sqrt{(\tauswitch - \tau_{\min}) \log \left(\frac{1}{\delta}\right)}\right)$ with probability at least $1 - \delta$.
Putting all of this together gives us
\begin{align*}
|\W(\tauswitch)| &\geq 0.5 (\tauswitch - \tau_{\min}) - \Oh\left(\sqrt{(\tauswitch - \tau_{\min}) \log \left(\frac{1}{\delta}\right)}\right) = \Omega(\tauswitch).
\end{align*}
% \vm{The fact that $1 - \nu_t \geq 1/2$ is because we only need $\tau_{\min} \geq 8$, which will be the case from the definition. I can say that explicitly if needed.}
% \ak{I think first inequality is not true, since before $\tau_{\min}$ we may have $Y_t = 0$ and $Z_t=1$ in which case we do not add the point to $\W$? More generally, the semantics of $Z_t$ are pretty confusing.}
% where the last inequality follows for large enough $T$, and the second-to-last inequality is a consequence of the Hoeffding bound.
Plugging this into Lemma~\ref{lem:metaanalysis} completes the proof of the theorem.
% \vm{Is this somewhat improved?}

\paragraph{Case 2: Action-averaged feature diversity holds}
In the case where action-averaged feature diversity holds, it suffices to provide an upper bound on $|\T(T)|$ and a lower bound on $|\W(\tauswitch)|$.
We will not define any extra statistical events for this case.
First, we note that because $Z_t \in \{0,1\}$, we can apply Hoeffding's inequality to get
\begin{align*}
    |\T(T)| \leq \sum_{t=1}^T (1 - Z_t) \leq 2 \sum_{t=1}^T \nu_t \leq 4 T^{2/3}
\end{align*}
with probability at least $1 - e^{-T^{1/3}}$.
This gives us SM regret that scales as $\Oh(T^{2/3})$.
(We completely sacrifice on the instance-dependent guarantees in this case.)

Next, we characterize $|\W(\tauswitch)|$.
We have
\begin{align*}
|\W(\tauswitch)| &= \sum_{t=1}^{\tauswitch} Y_t Z_t + (1 - Y_t) (1 - Z_t),
\end{align*}
and since $\{Z_t\}_{t \geq 1}$ is an iid sequence completely independent of $\{Y_1,\ldots,Y_{\tauswitch}\}$, we have
\begin{align*}
\E{\W(\tauswitch)| \{Y_1,\ldots,Y_{\tauswitch}\}} &= \E{\sum_{t=1}^{\tauswitch} (1 - Z_t) + Y_t (2Z_t - 1)|\{Y_1,\ldots,Y_{\tauswitch}\}} \\
&= \sum_{t=1}^{\tauswitch} \sum_{t=1}^{\tauswitch} t^{-1/3} + Y_t (1 - 2t^{-1/3}) \geq \sum_{t=1}^{\tauswitch} t^{-1/3} = \Omega({\tauswitch}^{2/3}).
\end{align*}
Applying the tower property of conditional expectations then gives us $\E{|\W({\tauswitch})|} = \Omega(({\tauswitch})^{2/3})$.
The last inequality uses the fact that $Y_t \geq 0$ and $1 - 2t^{-1/3} > 0$, therefore that term can be lower bounded by $0$.

This tells us that it suffices to show that $\W({\tauswitch})$ concentrates around its expectation.
For this, we note that $Y_t Z_t + (1 - Y_t) (1 - Z_t)$ is bounded between $0$ and $2$ and use the Azuma-Hoeffding inequality to get $|\W(\tauswitch)| \geq \frac{1}{2} \E{|\W(\tauswitch)|}$ with probability at least $1 - e^{-(\tauswitch)^{1/3}/2}$.
It suffices to consider $\tauswitch \geq (\log T)^3$ (as otherwise, we would just have $R_T^C \leq (\log T)^3$).
Then, we get $1 - e^{-(\tauswitch)^{1/3}/2} \geq 1 - \frac{1}{\sqrt{T}}$, and so we have $|\W(\tauswitch)| \geq \frac{1}{2} \E{|\W(\tauswitch)|}$ with high probability as desired.

Plugging these bounds into Lemma~\ref{lem:metaanalysis} gives us
\begin{align*}
    R_T^C &= \Oh\left(T^{2/3} + \sqrt{dT} + \frac{1}{\gamma^2} \cdot \sqrt{dT} + \frac{d^{1/4} (\tauswitch)^{2/3}}{\gamma^2}\right) \\
    &\leq \Oh\left(T^{2/3} + \sqrt{dT} + \frac{1}{\gamma^2} \cdot \sqrt{dT} + \frac{d^{1/4} T^{2/3}}{\gamma^2}\right),
\end{align*}
which completes the proof for this case.
\end{proof}

\section{Partial results for model selection in nested linear contextual bandits}\label{sec:linearCB}

In this section, we provide a simple extension of our approach to universal model selection to the case of multiple linear contextual bandits (linear CB).
We use the setup that is provided in~\cite{foster2019model}, i.e. the reward associated with action $i$ at round $t$ is given by
\begin{align}\label{eq:generativemodellinearCB}
g_{i,t} = \lin{\x_{i,t},\Tstar} + W_{i,t},
\end{align}
where $\Tstar \in \R^{d_M}$, and $d_M$ denotes the maximal possible dimension of the model (which is also the dimension of the provided contexts $\x_{i,t}$).
While this assumes that the rewards are \emph{realizable} under this maximal dimensional model, there may be hidden simpler structure that is unknown a-priori.
To model this simpler structure, we consider $0 < d_1 < \ldots < d_M$, and stipulate that the rewards are realizable under \emph{model order} $j^* \in [M]$ if only the first $d_{j^*}$ coordinates of $\Tstar$ are non-zero.
This means that we can represent the rewards as
\begin{align}
g_{i,t} = \lin{\x_{i,t,d_{j^*}},\Tstar_{(d_{j^*})}} + W_{i,t}
\end{align}
where $\Tstar_{(d_{j^*})}$ denotes the first $j^*$ coordinates of $\Tstar$ and $\x_{i,t,d_{j^*}}$ denotes the first $d_{j^*}$ coordinates of the context $\x_{i,t}$.

In this setup, Objective 1 of the model selection problem corresponds to achieving a rate of the form $R_T^{j^*} = \widetilde{\Oh}(\sqrt{d_{j^*} T})$, and Objective 2 corresponds to achieving a rate of the form $R_T^{j^*} = \widetilde{\Oh}(d_{j^*}^{\alpha} T^{1 - \alpha})$.
Here, $j^*$ is the true model order and $R_T^{j^*}$ denotes the regret with respect to the best parameter 
Like in the simpler bandit-vs-CB problem,~\cite{foster2019model} achieves Objective 2 for the case $\alpha = 1/3$, but only under the condition that the action-averaged covariance matrix $\Sigmabold$ is well-conditioned, i.e. we have $\Sigmabold \succeq \gamma_{\min} \eye$ for a constant $\gamma_{\min}$ that does not depend on $d$ or $T$.
The hidden factors in the model selection rates $\widetilde{\Oh}(d_{j^*}^{\alpha} T^{1 - \alpha})$ scale inversely with $\gamma_{\min}$.

\begin{algorithm}
    \caption{ EstimateResidualMultipleModels} \label{alg:estimateresidualsideinfo2}
    \begin{algorithmic}
        \STATE \textbf{Input}: Examples $\{(\x_s,y_s)\}_{s=1}^n$ and second moment matrix estimates $\Sigmahat \in \R^{d \times d}$ and $\Sigmahat_1 \in \R^{d_1 \times d_1}$.
        \STATE Return estimator
        \begin{align*}
        \Sgaphat := \frac{1}{\binom{n}{2}} \sum_{s < t} \Big{\langle} \Sigmahat^{1/2} \Rhat \x_s y_s, \Sigmahat^{1/2} \Rhat \x_t y_t) \Big{\rangle}
        \end{align*}
        of the square-loss gap $\Sgap := \E{(\x^\top \Tstar - \x_1^\top \Tstar_1)^2}$, where $\Rhat := \begin{bmatrix} \Omegahat_1 & \boldsymbol{0} \\ \boldsymbol{0} & \boldsymbol{0} \end{bmatrix} - \Omegahat$.
        \end{algorithmic}
\end{algorithm}

For this extension, we directly use the framework of \textsc{ModCB}~\citep{foster2019model} and do not reproduce the details here, except for the square loss gap estimator provided above in Algorithm~\ref{alg:estimateresidualsideinfo2} for two candidate dimensions $d_1 < d$.
Here, $d_1 := d_{i-1}$ and $d := d_{i}$ for some value of $i \in [M]$, and so this routine will be used at the $i^{th}$ stage of model selection, i.e deciding between model orders $i$ and $i+1$.
Note that here, we can write the square loss gap in the form
\begin{align}\label{eq:gapexpressionlinearCB}
\Sgap = (\Rmat \Sigmabold \Tstar)^\top \Sigmabold (\Rmat \Sigmabold \Tstar), 
\end{align}
where we define $\Rmat := \begin{bmatrix} \Omegabold_1 & \boldsymbol{0} \\ \boldsymbol{0} & \boldsymbol{0} \end{bmatrix} - \Omegabold$.

Our extension is simple to describe in this context: we simply use the estimators $\Sigmahat := \Tgamma{\Sigmahat_t}$ and $\Sigmahat_1 := \Tgamma{\Sigmahat_1^{(t)}}$ with the subroutine provided in Algorithm~\ref{alg:estimateresidualsideinfo2}.
We require an extra assumption of \emph{block-diagonal structure} on the covariance matrix, which we provide below.

\begin{definition}\label{as:blockdiagonal}
We use the notation $d_1 := d_{i-1}$ and $d := d_{i}$ as above.
Then, for each value of $i \in [M]$ the action-averaged covariance matrix $\Sigmabold$ is assumed to possess block-diagonal structure of the form
\begin{align}
\Sigmabold := \begin{bmatrix} 
\Sigmabold_1 & \boldsymbol{0} \\
\boldsymbol{0} & \Sigmabold_{-1}
\end{bmatrix}.
\end{align}
Note that under this block-diagonal assumption, we simply get $\Rmat = \begin{bmatrix} \boldsymbol{0} & \boldsymbol{0} \\ \boldsymbol{0} & \Sigmabold_{-1} \end{bmatrix}$.
\end{definition}
The assumption in Definition~\ref{as:blockdiagonal} is utilized to avoid blow-ups in the approximation error arising from thresholding due to cross-correlation terms.
It will not, in general, hold for a linear contextual bandit problem; therefore, it is an important question for future work to provide universal model selection rates in the absence of this assumption.
Nevertheless, under this assumption we can extend the universal model selection result of Theorem~\ref{thm:universal} to the case of model selection among nested linear CB models, as stated below.
\begin{theorem}\label{thm:universallinearCB}
Algorithm~\ref{alg:metamethod} with the residual estimator given in Algorithm~\ref{alg:estimateresidualsideinfo2} and the inverse covariance matrix estimate that uses eigenvalue thresholding with the choice $\gamma_j := (d_j/T)^{1/3}$ corresponding to model order $j$ achieves model selection rate
\begin{align}\label{eq:universallinearCB}
R_T^{j^*} &= \widetilde{\Oh}_{\delta}(d_{j^*}^{1/6}T^{5/6}).
\end{align}
Here, $j^*$ is the minimal model order under which the rewards are realizable, $\delta \in (0,1)$ denotes a failure probability, and the $\widetilde{\Oh}_{\delta}(\cdot)$ hides sublinear dependences in $K$ and polylogarithmic dependences on $1/\delta$.
\end{theorem}
We conclude this section with a proof of Theorem~\ref{thm:universallinearCB}.
The following lemma (intended to replace Theorem 2 in~\citet{foster2019model}) characterizes the estimation error $|\Sgaphat - \Sgap|$ will depend on the choice of $\gamma$ for $d_1 < d$.

\begin{lemma}\label{lem:truncatedestimatorlinearCB}
We use the notation $d_1$ and $d$ from Definition~\ref{as:blockdiagonal} to denote the currently estimated model dimension and the next model dimension respectively.
Suppose that we have $n$ labeled samples and $m$ unlabeled samples, and $m > n$.
Provided that $m \geq C(d + \log(2/\delta))/\gamma$, the estimator provided in Algorithm~\ref{alg:estimateresidualsideinfo} guarantees that
\begin{align}\label{eq:esterrorratelinearCB}
    |\Sgaphat - \Sgap| &\leq \frac{1}{2} \Sgap + \mathcal{O}\left(\frac{1}{\gamma^2} \cdot \frac{d^{1/2} \log^2(2d/\delta)}{n} + \frac{1}{\gamma^4} \cdot \eu{\E{\x y}}_2^2 \cdot \frac{d + \log(2/\delta)}{m} +\gamma \right) 
\end{align}
with probability at least $1 - \delta$.
\end{lemma}
Thus, we can utilize Lemma~\ref{lem:truncatedestimatorlinearCB} with different thresholding values $\gamma_i$ corresponding to each estimated model class $i \in [M]$.
In particular, it is clear that in a manner similar to the proof of Theorem~\ref{thm:universal}, the choice $\gamma_i := (d_i/T)^{1/3}$ will yield the desired model selection rate.
% \vm{This is the ``sketch'' that I was alluding to earlier. Will see whether worth fleshing out for arXiv.}
Therefore, we only provide the proof of Lemma~\ref{lem:truncatedestimatorlinearCB} here.
\begin{proof}
The proof follows along similar lines to the proof of Lemma~\ref{lem:truncatedestimator}, but with several extra terms.
Before beginning the proof, we define a term called the \textit{truncated square-loss gap} as below:
\begin{align}\label{eq:truncatedsquareloss}
   \Sgaptilde := (\Rgamma \Sigmabold \Tstar)^\top \Tgamma{\Sigmabold} (\Rgamma  \Sigmabold \Tstar) .
   % \Sgaptilde := (\Sigmabold \Tstar)^\top \Tgamma{\Sigmabold}^{-1} \Sigmabold \Tstar.
\end{align}
where we define $\Rgamma := \begin{bmatrix} \Tgamma{\Sigmabold_1}^{-1} & \boldsymbol{0} \\ \boldsymbol{0} & \boldsymbol{0} \end{bmatrix} - \Tgamma{\Sigmabold}^{-1}$ as shorthand.
It is easy to see that $\Rgamma = \begin{bmatrix} \boldsymbol{0} & \boldsymbol{0} \\ \boldsymbol{0} & \Tgamma{\Sigmabold_{-1}}^{-1} \end{bmatrix}$.
We also recall that we defined $\Omegabold := \Sigmabold^{-1}$ and $\Omegahat := \Sigmahat^{-1}$. 
Finally, we define $\mubold := \E{\x y} = \Sigmabold \Tstar$ as shorthand.

The proof is carried out in three distinct steps:
\begin{enumerate}
    \item Upper-bounding $\left|\Sgaphat - \E{\Sgaphat}\right|$, the ``variance" estimation error arising from $n$ samples.
    \item Upper-bounding $\left|\E{\Sgaphat} - \Sgaptilde\right|$, the bias-term with respect to the truncated squared loss gap.
    \item Upper-bounding $\left|\Sgaptilde - \Sgap\right|$, the bias arising from truncation.
\end{enumerate}

\noindent \textit{1. Upper-bounding $\left|\Sgaphat - \E{\Sgaphat}\right|$.}
We note that $\E{\Sgaphat} = \eu{\Sigmahat^{1/2}\Rhat \mubold}_2^2$.
We consider the random vector
\begin{align*}
    \Sigmahat^{1/2}\Rhat \x y - \Sigmahat^{1/2} \Rhat \E{\x y} ,
\end{align*}
and show that it is sub-exponential with parameter $\mathcal{O}(2/ \gamma)$.
This follows because $\eu{\Sigmahat^{1/2} \Rhat}_{\op} \leq \eu{\Omegahat^{1/2}}_{\op} + \eu{\Omegahat_1^{1/2}}_{\op} \leq \frac{2}{\sqrt{\gamma}}$, where the last inequality follows by the definition of the truncation operator.
% \vm{Note: we have used the block-diagonal structure to obtain the first triangle inequality. More generally there would be a slightly worse dependence on $\gamma$.}

Thus, using the sub-exponential tail bound just as in Lemma 17,~\cite{foster2019model}, we get 
\begin{align*}
    \left|\Sgaphat - \E{\Sgaphat}\right| &= \mathcal{O}\left(\frac{1}{\gamma} \cdot \frac{d^{1/2} \log^2(2d/\delta)}{n} + \frac{1}{\sqrt{\gamma}} \cdot \frac{\eu{\Sigmahat^{1/2}\Rhat \E{\x y})}_{2} \log(2/\delta)}{\sqrt{n}}\right) .
\end{align*}

Now, we note that $\eu{\Sigmahat^{1/2} \Rhat \E{\x y}}_{2} = \sqrt{\E{\Sgaphat}}$.
Therefore, we apply the AM-GM inequality to deduce that
\begin{align}\label{eq:prooffirstpartlinearCB}
    \left|\Sgaphat - \E{\Sgaphat}\right| &\leq \frac{1}{8} \E{\Sgaphat} +  \mathcal{O}\left(\frac{1}{\gamma} \cdot \frac{d^{1/2} \log^2(2d/\delta)}{n} \right)
\end{align}

\noindent \textit{2. Upper-bounding $\left|\E{\Sgaphat} - \Sgaptilde\right|$.} We know that $\E{\Sgaphat} = \langle \Rhat \Sigmahat \Rhat \mubold, \mubold \rangle$ and \newline $\Sgaptilde = \langle \Rgamma \Tgamma{\Sigmabold} \Rgamma \mubold, \mubold \rangle$.
Following an identical sequence of steps to~\citet{foster2019model} (in particular, the steps leading up to Eq. (16) in Appendix C.1), we get
\begin{align*}
    \left|\E{\Sgaphat} - \Sgaptilde\right| &\leq \frac{1}{8} \Sgaptilde + \mathcal{O}\Big{(}\eu{\Tgamma{\Sigmabold}^{-1/2} \Sigmahat (\Rgamma - \Rhat) \mubold }_2^2 \\
    &+ \eu{\Tgamma{\Sigmabold}^{-1/2} (\Sigmahat - \Tgamma{\Sigmabold}) \Rhat \mubold}_2^2 + \eu{\Sigmahat^{1/2}(\Rgamma - \Rhat) \mubold}_2^2 \Big{)}
\end{align*}
Furthermore, note that 
\begin{align*}
\Rhat - \Rgamma = (\Omegahat - \Tgamma{\Sigmabold}^{-1}) - \begin{bmatrix} \Omegahat_1 - \Tgamma{\Sigmabold_1}^{-1} & \boldsymbol{0} \\ \boldsymbol{0} & \boldsymbol{0} \end{bmatrix}.
\end{align*}
We now state and prove the following lemma, which will help us characterize the sum above.

\begin{lemma}\label{lem:opnormcontrollinearCB}
We have
\begin{align}\label{eq:opnormcontrollinearCB}
    \eu{\Tgamma{\Sigmabold}^{-1/2} \Sigmahat \Tgamma{\Sigmabold}^{-1/2} - \eye}_{\op} &= \eu{\Tgamma{\Sigmabold_1}^{-1/2} \Sigmahat_1 \Tgamma{\Sigmabold_1}^{-1/2} - \eye}_{\op} = \Oh\left(\frac{\epsilon}{\gamma}\right) \text{ where } \\
    \epsilon &= \sqrt{\frac{d + \log(2/\delta)}{m}} \nonumber.
\end{align}
with probability at least $1 - \delta$.
From the above, we get the following inequalities:
\begin{subequations}\label{eq:opnormequationslinearCB}
\begin{align}
    \eu{\Tgamma{\Sigmabold}^{-1/2} \Sigmahat}_{\op} &\leq \mathcal{O}(1) \label{eq:opnormequationsa} \\
    \eu{\Rhat \mubold}_{2} &\leq \mathcal{O}(1/\gamma) \label{eq:opnormequationsb} \\
    \eu{\Rgamma - \Rhat}_{\op} &\leq \mathcal{O}(\epsilon/\gamma^2) \label{eq:opnormequationsc}
\end{align}
\end{subequations}
\end{lemma}
We will prove Lemma~\ref{lem:opnormcontrollinearCB} at the end of this proof.
Before that, we use it to complete the proof of Lemma~\ref{lem:truncatedestimator}.
Using Equation~\eqref{eq:opnormequationslinearCB} and the sub-multiplicative property of the operator norm, we get
\begin{align*}
    \eu{\Tgamma{\Sigmabold}^{-1/2} \Sigmahat (\Rgamma - \Rhat) \mubold}_2^2 &\leq \eu{\Tgamma{\Sigmabold}^{-1/2} \Sigmahat}_{\op}^2 \eu{\mubold}_2^2 \cdot \eu{\Rgamma - \Rhat}_{\op}^2 \leq \mathcal{O}\left(\frac{\eu{\mubold}_2^2 \cdot \epsilon^2}{\gamma^4}\right) \\
    \eu{\Tgamma{\Sigmabold}^{-1/2} (\Sigmahat - \Tgamma{\Sigmabold}) \Rhat \mubold}_2^2 &\leq \eu{\Tgamma{\Sigmabold}}_{\op} \eu{\Rhat \mubold}_2^2 \cdot \eu{\Tgamma{\Sigmabold}^{-1/2} \Sigmahat \Tgamma{\Sigmabold}^{-1/2} - \eye}_{\op}^2 \\
    &\leq \mathcal{O}\left(\frac{\eu{\mubold}_2^2 \cdot \epsilon^2}{\gamma^4}\right) \\
    \eu{\Sigmahat^{1/2}(\Rgamma - \Rhat) \mubold}_2^2 &\leq \eu{\Sigmahat}_{\op} \eu{\mubold}_2^2 \cdot \eu{\Rgamma - \Rhat}_{\op}^2 \leq \mathcal{O}\left(\frac{\eu{\mubold}_2^2 \cdot \epsilon^2}{\gamma^4}\right)
\end{align*}
This directly gives us
\begin{align}\label{eq:proofsecondpartlinearCB}
    \left|\E{\Sgaphat} - \Sgaptilde\right| \leq \frac{1}{8} \Sgaptilde + \mathcal{O}\left(\frac{\eu{\mubold}_2^2 \cdot \epsilon^2}{\gamma^4}\right) .
\end{align}

\noindent \textit{3. Upper-bounding $\left|\Sgaptilde - \Sgap\right|$.}
Observe that 
\begin{align*}
    \Sgaptilde &= (\Tstar)^\top \Sigmabold \Rgamma \Tgamma{\Sigmabold} \Rgamma \Sigmabold \Tstar \text{ and }  \\
    \Sgap &=  (\Tstar)^\top \Sigmabold \Rmat \Sigmabold \Rmat \Tstar .
\end{align*}

This directly implies that
\begin{align*}
    |\Sgaptilde - \Sgap| &= |(\Tstar)^\top \Sigmabold (\Rgamma \Tgamma{\Sigmabold} \Rgamma  - \Rmat \Sigmabold \Rmat)\Sigmabold \Tstar| \\
    &\leq \eu{\Sigmabold (\Rgamma \Tgamma{\Sigmabold} \Rgamma - \Rmat \Sigmabold \Rmat) \Sigmabold}_{\op},
\end{align*}
where the second inequality follows because we have assumed bounded signal, i.e. $\eu{\Tstar}_2 \leq 1$.

It remains to control the operator norm terms above.
We denote $\Sigmabold := \Ubold \Lambdabold \Ubold^\top$, and write $\Lambdabold := \begin{bmatrix} \Lambdabold_1 & \boldsymbol{0} \\ \boldsymbol{0} & \Lambdabold_{-1} \end{bmatrix}$.
First, we note that because of the block-diagonal structure considered in Definition~\ref{as:blockdiagonal}, we can also write $\begin{bmatrix} \Sigmabold_1 & \boldsymbol{0} \\ \boldsymbol{0} & \boldsymbol{0} \end{bmatrix} := \Ubold \begin{bmatrix} \Lambdabold_1 & \boldsymbol{0} \\ \boldsymbol{0} & \boldsymbol{0} \end{bmatrix} \Ubold^\top$.
This directly gives us $\Rmat \Sigmabold \Rmat = \Ubold \begin{bmatrix} \boldsymbol{0} & \boldsymbol{0} \\ \boldsymbol{0} & \Lambdabold_{-1}^{-1} \end{bmatrix} \Ubold^\top$.

Second, we note that $\Tgamma{\Sigmabold}^{-1} := \Ubold \Tgamma{\Lambdabold}^{-1} \Ubold^\top$.
Recalling the definition of $\Rgamma$ and using again the block-diagonal structure in Definition~\ref{as:blockdiagonal}, simple algebra gives us $$\Rgamma = \Ubold \begin{bmatrix} \boldsymbol{0} & \boldsymbol{0} \\ \boldsymbol{0} & \Tgamma{\Lambdabold_{-1}}^{-1} \end{bmatrix} \Ubold^\top$$ and $\Rgamma \Tgamma{\Sigmabold} \Rgamma = \Ubold \begin{bmatrix} \boldsymbol{0} & \boldsymbol{0} \\ \boldsymbol{0} & \Tgamma{\Lambdabold_{-1}}^{-1} \end{bmatrix} \Ubold^\top$.
Putting all of this together directly gives us
\begin{align*}
\eu{\Sigmabold (\Rgamma \Tgamma{\Sigmabold} \Rgamma - \Rmat \Sigmabold \Rmat) \Sigmabold}_{\op} = \eu{\Lambdabold_{-1} \Tgamma{\Lambdabold_{-1}}^{-1} \Lambdabold_{-1} - \Lambdabold_{-1}}_{\op} \leq \gamma
\end{align*}
and so, we get
\begin{align}\label{eq:proofthirdpartlinearCB}
    \left|\Sgaptilde - \Sgap\right| \leq \gamma .
\end{align}
Thus, putting together Equations~\eqref{eq:prooffirstpartlinearCB},~\eqref{eq:proofsecondpartlinearCB} and~\eqref{eq:proofthirdpartlinearCB} completes the proof.

It only remains to prove Lemma~\ref{lem:opnormcontrollinearCB}, which we now do.
\begin{proof}
First, we note that Equation~\eqref{eq:opnormcontrollinearCB} directly implies Equations~\eqref{eq:opnormequationsa} and~\eqref{eq:opnormequationsb}.
To see this, we note the following in order:
\begin{enumerate}
    \item We have $\eu{\Tgamma{\Sigmabold}^{-1/2} \Sigmahat}_{\op} \leq \eu{\Tgamma{\Sigmabold}^{-1/2} \Sigmahat \Tgamma{\Sigmabold}^{-1/2}}_{\op} \eu{\Tgamma{\Sigmabold}}_{\op} \leq \left(1 + \frac{\epsilon}{\gamma}\right) = \mathcal{O}(1)$. This shows that Equation~\eqref{eq:opnormequationsa} holds.
    \item We have $\eu{\Rhat}_{\op} \leq \eu{\Sigmahat^{-1}}_{\op} = \mathcal{O}(1/\gamma)$. This shows that Equation~\eqref{eq:opnormequationsb} holds.
    % \item We have $\eu{\Sigmahat^{-1} - \Tgamma{\Sigmabold}^{-1}}_{\op} \leq \eu{\Tgamma{\Sigmabold}^{-1}}_{\op} \cdot \eu{\Tgamma{\Sigmabold}^{-1/2} \Sigmahat \Tgamma{\Sigmabold}^{-1/2} - \eye}_{\op} \leq \epsilon/\gamma^2$. 
    % A similar inequality holds for $\eu{\Sigmahat_1^{-1} - \Tgamma{\Sigmabold_1}^{-1}}_{\op}$.
    % Therefore, we get $\eu{\Rgamma - \Rhat}_{\op} \leq \eu{\Sigmahat^{-1} - \Tgamma{\Sigmabold}^{-1}}_{\op} + \eu{\Sigmahat_1^{-1} - \Tgamma{\Sigmabold_1}^{-1}}_{\op} \leq 2\epsilon/\gamma^2$.
    % This shows that Equation~\eqref{eq:opnormequationsc} holds.
\end{enumerate}
Moreover, Equation~\eqref{eq:opnormequationsc} follows directly from Lemma~\ref{lem:opnormcontrol}.
Therefore, it suffices to prove Equation~\eqref{eq:opnormcontrollinearCB} here.
We only provide the argument for $\Sigmahat$ with respect to $\Tgamma{\Sigmabold}$ here: an identical argument holds to bound $\Sigmahat_1$ with respect to $\Tgamma{\Sigmabold}$.
First, recall that $\Sigmahat := \Tgamma{\Sigmahat_m}$, and so we really want to upper bound the quantity $\eu{\Tgamma{\Sigmahat_m} - \Tgamma{\Sigmabold}}_{\op}$.
It is well known (see, e.g.~\cite{boyd2004convex}) that for any positive semidefinite matrix $\M$, the operator $\Tgamma{\M} - \gamma \eye $ is a proximal operator with respect to the convex nuclear norm functional.
The non-expansiveness of proximal operators then gives us
\begin{align*}
    \eu{\Sigmahat - \Tgamma{\Sigmabold}}_{\op} &= \eu{\Tgamma{\Sigmahat_m}- \Tgamma{\Sigmabold}}_{\op} \\
    &= \eu{\Tgamma{\Sigmahat_m} - \gamma \eye - ( \Tgamma{\Sigmabold} - \gamma \eye)}_{\op} \\
    &\leq \eu{\Sigmahat_m - \Sigmabold}_{\op} \leq \epsilon ,
\end{align*}

where the last step follows by standard arguments on the concentration of the empirical covariance matrix.
Thus, we have
\begin{align*}
    \eu{\Tgamma{\Sigmabold}^{-1/2} \Sigmahat \Tgamma{\Sigmabold}^{-1/2} - \eye}_{\op} \leq \eu{\Tgamma{\Sigmabold}^{-1}}_{\op} \cdot \eu{\Sigmahat - \Tgamma{\Sigmabold}}_{\op} = \frac{1}{\gamma} \cdot \epsilon = \frac{\epsilon}{\gamma} .
\end{align*}
This completes the proof.
\end{proof}

With the proof of Lemma~\ref{lem:opnormcontrollinearCB} complete, we have completed the proof of Lemma~\ref{lem:truncatedestimator}.
\end{proof}

\section{Technical lemmas}\label{sec:technical}

\subsection{Proof of Lemma~\ref{lem:sigmabar_t}}

It is equivalent to show that $\E{\sum_{j=1}^{n} \x_{s_j,A_{s_j}} \x_{s_j,A_{s_j}}^\top} \preceq n \eye$.
We will critically use the fact that on all designated exploration rounds $1 \leq s_1 \leq \ldots \leq s_n := \tau(n)$, the distribution of the action $A_s$ is independent of the context realizations at round $s$, and in particular independent of $\x_{s_j,A_{s_j}}$.
Thus, we use the tower property of conditional expectations to get
\begin{align*}
\E{\sum_{j=1}^{n} \x_{s_j,A_{s_j}} \x_{s_j,A_{s_j}}^\top} &= \E{\sum_{j=1}^n \E{\x_{s_j,A_{s_j}} \x_{s_j,A_{s_j}}^\top \Big{|} A_{s_j}}} 
\preceq \E{\sum_{j=1}^n \eye} 
= n \eye.
\end{align*}
Above, the inequality follows because we have $\Sigmabold_i \preceq \eye$ for all $i \in [K]$.
This completes the proof. 
On the other hand, we have $\Sigmabold_{A_{s_j}} \succeq \gamma \eye$ for all $j \in [n]$ by the definition of the data-adaptive exploration set (which maintains feature diversity).
% \ak{What about the lower bound of $\gamma \eye$?}  \vm{This follows from the definition of the exploration set; should I say it explicitly?}

%!TEX root = main.tex

\section{Data-adaptive exploration for general contextual bandits: A partial roadmap}\label{sec:generalroadmap}

In this section, we discuss a partial roadmap to extend the data-adaptive exploration sub-routine in \textsc{ModCB.A} to general contextual bandits by leveraging ideas from function estimation under covariate-shift.
Our starting point is the observation that the optimal model selection algorithm (OSOM)~\cite{chatterji2020osom} that achieves Objective 1 does not require a fast estimation subroutine; only a favorable diversity condition that is purely \emph{covariate-dependent}.
We present a generalization of this arm-specific diversity condition to general function classes that is motivated by considerations in covariate shift and transfer learning.
Moreover, since the data-adaptive exploration subroutine in \textsc{ModCB.A} tests for the presence of a feature diversity condition, its principle could be generalized as well.
In particular, we could adaptively decide whether to deploy the ``greedy'' action selection used by OSOM or a forced-exploration subroutine that is exclusively used to estimate the putative regret bound~\cite{pacchiano2020model,lee2021online}.
Of course, such an adaptive procedure would not achieve even Objective 2 in the worst case --- indeed, Objective 2 was shown to not be achievable for general function classes by any algorithm~\cite{marinov2021pareto}.
However, it opens the door to achieving Objective 1 under favorable diversity conditions when they exist --- a more realistic, and still desirable goal.

\subsection{From benign covariate-shift to optimal model selection}

We consider the bandit-vs-contextual bandit model selection formulation, and a general function class $\F$ such that the reward model is given by 
\begin{align}\label{eq:rewardmodelgeneral}
G_{i,t} = \mu_i + f^*(\x_{i,t}) + W_{i,t} \text{ for every } i \in [K]
\end{align}
for some (unknown) $f^* \in \F$ and bias terms $\{\mu_i\}_{i \in [K]}$.
To preserve the property of nestedness, we assume that the zero function (i.e. $f_0(\x) = 0 \text{ for all } \x \in \R^d$) is contained in the function class $\F$.
We also assume, without loss of generality, that all functions in $\F$ are \emph{centered} in the sense that $\E{f(\x_i)} = 0$ for all $f \in \F$ and all $i \in [K]$.
We consider \emph{learnable} function classes, in the sense that it is information-theoretically possible to design a contextual bandit algorithm that achieves regret $\Oh(\sqrt{\text{comp}(\F) \cdot T})$ with respect to the best policy induced by $f^*(\cdot)$, where $\text{comp}(\F)$ is a standard learning-theoretic measure of function complexity\footnote{For example, in the case of linear models we have $\text{comp}(\F) = d$ and in the case of unstructured finite function classes, we have $\text{comp}(\F) = \log |\F|$.}.
Thus, Objective 1 of model selection would be to simultaneously achieve regret rates $\Oh(\sqrt{KT})$ under simple MAB structure and $\Oh(\sqrt{K \cdot \text{comp}(\F) \cdot T}$ more generally (i.e. under Equation~\eqref{eq:rewardmodelgeneral}).

We first observe that, in the case of linear function classes, arm-specific feature diversity is a sufficient condition for reliable estimation under covariate-shift across any two arms.
This connection is made precise in Example~\ref{eg:linear} and allows us to postulate a more general feature-diversity condition.
Suppose that for each arm $i \in [K]$, we have $\x_{i,t} \text{ i.i.d. } \sim \D_i$ across $t \geq 1$.
Then, a favorable situation for transfer learning would allow an estimator constructed from samples under the distribution $\D_i$ to generalize \emph{only upto a constant factor worse} on samples obtained from the shifted distribution $\D_j$; further, we would like this property to hold for any $j \neq i$.
The definition below formalizes such a sufficient condition in terms of the function class $\F$ and data distribution tuple $(\D_1,\ldots,\D_K)$.
\begin{definition}\label{def:covariateshift}
A function class $\F$ is said to be covariate-agnostic with factor $C > 1$ with respect to data distributions $\D_1,\ldots,\D_K$ if, for every $i \neq j \in [K]$ and any two functions $f,f' \in \F$, we have
\begin{align}\label{eq:covariateshift}
\E{(f(\x_j) - f'(\x_j))^2} \leq C \cdot \E{(f(\x_i) - f'(\x_i))^2}.
\end{align}
Here, $C$ should be a constant that does not directly or indirectly depend on the model complexity measure $\text{comp}(\F)$, and only depends on the data distribution tuple $(\D_1,\ldots,\D_K)$.
\end{definition}
Intuitively, the condition of covariate-agnosticity as defined above should reduce the requirement for forced exploration, raising the possibility of achieving Objective 1 through ``greedy'' action selection, i.e. $A_t = i_t$ under the MAB hypothesis.
This is because samples collected under the action sequence recommended by a MAB algorithm, i.e. $\{A_t = i_t\}_{t=1}^T$, could be used to estimate the putative regret bound that would be obtained through the counterfactual action sequence recommended by a CB algorithm, i.e. $\{A_t = j_t\}_{t=1}^T$, with only a constant factor inflation in estimation error.
Indeed, this is the approach at the heart of the \textsc{OSOM} algorithm~\cite{chatterji2020osom}, and the reason why a feature diversity condition is required for it to work.
This suggests a plausible extension of this approach to general function classes under data distributions that satisfy Equation~\eqref{eq:covariateshift}.
Formally establishing such an extension\footnote{While feature-diversity in the sense of Equation~\eqref{eq:covariateshift} is the most essential requirement, we note that there are several other non-trivial technical pieces required for such a generalization to work, most notably the ability to obtain self-normalized generalization bounds from adaptively collected samples (well-known for the linear case~\cite{pena2008self,abbasi2011improved}) under general function classes. While this may constrain the set of function classes we can work with, we still expect such bounds to be establishable well beyond the linear case.} is an interesting direction for future work.

%%%FLESH OUT AT SOME LATER DATE
% Under the condition of covariate-agnosticity defined above, we expect that a generalization of the OSOM algorithm could be shown to achieve Objective 1 with access to a self-normalizing estimator of $f^*$.
% This would constitute 
% We would require the following statistical learning theoretic condition --- access to a sequence of estimators $\widehat{f}_t(\cdot)$ of the function $f^*$ collected from samples $\{\x_{i_s,s},g_{i_s,s}\}_{s =1}^{t-1}$ such that we can obtain high-probability error bounds (at least $1 - \delta$) on the quantity
% \begin{align*}
% \E{\frac{1}{t} \sum_{s=1}^t (\widehat{f}(\x_{i_s,s}) - f^*(\x_{i_s,s}))^2} \text{ for all } t \geq \tau_{\min}(\delta,\text{comp}(\F)).
% \end{align*}
% where $\tau_{\min}$ is some threshold that depends polylogarithmically on $\text{comp}(\F)$ and $1/\delta$.
% \todo{This is a self-bound as we will choose $A_s = i_s$ on all rounds. Note that if $A_s = i^*$ for all $s \geq 1$, i.e. passive sampling, this is a standard SLT guarantee.}
% \todo{How much do we need to flesh out the conditions above?}
% % \end{proposition}
% % \vm{Note: this is the other ingredient that is essential to optimal model selection, as it allows us to use the same adaptively collected samples used in a MAB algorithm to estimate $f^*$.}

\subsection{Adaptive exploration by testing for benign covariate-shift}

The property of covariate-agnosticity turns out to be equivalent to the arm-specific feature-diversity condition in the case of linear models.
Moreover, it can be characterized in several examples that go beyond linear function classes.
The examples provided below formally demonstrate this.
% Examples of covariate-agnostic function class-data distribution pairs are presented below, in an increasing order of complexity.

\begin{example}[Linear models]\label{eg:linear}
Consider the linear function class $\F := \{\langle \boldsymbol{\theta}, \cdot \rangle \text{ for all } \boldsymbol{\theta} \in \Theta\}$, and suppose that $\eye \succeq \Sigmabold_i \succeq \gamma \eye$ for all $i \in [K]$, which is precisely our definition of arm-specific feature diversity.
Then, an elementary calculation shows that Equation~\eqref{eq:covariateshift} holds for the choice $C = \frac{1}{\gamma}$.
Specifically, for two linear model parameters $\boldsymbol{\theta},\boldsymbol{\theta}'$, and any pair $i \neq j \in [K]$, we have
\begin{align*}
\E{\langle \x_j, \thetabold - \thetabold' \rangle^2} &= \|\Sigmabold_j^{1/2}(\thetabold - \thetabold')\|_2^2 \\
&= \|\Sigmabold_j^{1/2} \Sigmabold_i^{-1/2} \Sigmabold_i^{1/2} (\thetabold -\thetabold')\|_2^2 \\
&\leq \|\Sigmabold_j^{1/2} \Sigmabold_i^{-1/2}\|_{\op}^2 \cdot \E{\langle \x_i, \thetabold - \thetabold' \rangle^2} \\
&\leq \frac{1}{\gamma} \cdot \E{\langle \x_i, \thetabold - \thetabold' \rangle^2},
\end{align*}
where the last inequality follows because we have $\Sigmabold_i \succeq \gamma \eye$ and $\Sigmabold_j \preceq \eye$.
Moreover, equality is achieved when $\Sigmabold_j = \eye$ and $\gamma_{\min}(\Sigmabold_i) = \gamma$.
\end{example}

\begin{example}[Single-index models and high-dimensional data]\label{eg:singleindex}
The single-index model class can be modeled through functions of the form $f(\x) = \sum_{m = 1}^p a_m \langle \Tstar, \x \rangle ^p$, where $\{a_m\}_{m=1}^p$ may be known or unknown~\cite{dudeja2018learning}.
We consider the special case under which $p$ is a constant with respect to $d$.
% Suppose again that $\eye \succeq \Sigmabold_i \succeq \gamma \eye$ for all $i \in [K]$.
We assume that, for each $i \in [K]$, the distribution $\D_i$ on context $\x_i$ is sub-Gaussian with parameter at most $\Sigmabold_i$, and further satisfies the small-ball property~\cite{mendelson2014learning}, i.e. that there exists a universal positive constant $u > 0$ (that does not depend on $d$ or $p$) such that 
\begin{align}\label{eq:smallball}
\Pr\left[|\langle \Deltabold , \x_i \rangle|\geq u \cdot \|\Sigmabold_i^{1/2} \Deltabold\|_2 \right] \geq \frac{1}{2}
\end{align}
for any vector $\Deltabold$.
Finally, we assume that the arm-specific feature diversity condition holds, i.e. $\Sigmabold_i \succeq \gamma \eye$.
Then, we can verify that Equation~\eqref{eq:covariateshift} holds for the choice $C = \left(\frac{1}{u\gamma}\right)^p$.
(It suffices to consider the contribution from the highest-order terms (i.e. $a_m = 0$ for all $m < p$), and so we have
\begin{align*}
\E{(\langle \x_j, \thetabold\rangle^p - \langle \x_j, \thetabold'\rangle^p)^2} &\leq \|\Sigmabold_j^{1/2}(\thetabold - \thetabold')\|_2^{2p} \\
&= \|\Sigmabold_j^{1/2} \Sigmabold_i^{-1/2} \Sigmabold_i^{1/2} (\thetabold -\thetabold')\|_2^{2p} \\
&\leq \frac{1}{\gamma^p} \cdot \|\Sigmabold_i^{1/2} (\thetabold -\thetabold')\|_2^{2p} \\
&\leq \left(\frac{1}{u \gamma}\right)^p \cdot \E{(\langle \x_i, \thetabold\rangle^p - \langle \x_i, \thetabold'\rangle^p)^2}
\end{align*}
where the first inequality follows from sub-Gaussianity of $\D_j$ and the last inequality follows from the small-ball assumption on $\D_i$.)
\end{example}

\begin{example}[Nonparametric classes and bounded density ratios]\label{eg:nonparametric}
Finally, we consider $\F$ to be a general nonparametric function class under which consistent estimation is possible, such as the set of all Holder-smooth functions~\cite{nemirovski2000topics,audibert2007fast}.
For all $i \neq j \in [K]$, we further assume $\D_i$ to be absolutely continuous with respect to $\D_j$, and that $\sup_{\x \in \R^d} \frac{p_j(\x)}{p_i(\x)} \leq C_d$ (where $C_d$ may depend on the data dimension $d$).
Then, a special case of the results in~\cite{mansour2012multiple}, that express transfer error as a function of the Renyi divergence between distributions $\D_i$ and $\D_j$, shows\footnote{It is also possible to obtain transfer exponents for more general distributions, e.g. even when $\D_i$ is not absolutely continuous with respect to $\D_j$~\cite{kpotufe2021marginal}, but the resulting characterization of transfer error is significantly worse than Equation~\eqref{eq:covariateshift} and may therefore preclude optimal model selection.} that Equation~\eqref{eq:covariateshift} holds for $C := \log(C_d)$.
For low-dimensional data (i.e $d \ll \text{comp}(\F)$), the constant $\log(C_d)$ will be independent of the model complexity $\text{comp}(\F)$.
\end{example}
Examples~\ref{eg:linear} and~\ref{eg:singleindex} demonstrate that arm-specific feature diversity can constitute a sufficient test statistic for covariate-agnosticity (and the possibility of achieving Objective 1) in both linear and nonlinear models.
For these function classes, the data-adaptive exploration subroutine defined in Algorithm~\ref{alg:explorationschedule} could be directly plugged and played with alternative model selection algorithms designed for general function classes~\cite{lee2021online,pacchiano2020model}.
More generally, the multiplicative factor $C$ in Equation~\eqref{eq:covariateshift} can be characterized for very general function classes and data distributions~\cite{kpotufe2021marginal,hanneke2020value}.
% For absolutely continuous distributions, it typically suffices to consider a uniform upper-bound on the Radon-Nikodyn derivative (i.e. density ratio)~\cite{quinonero2008dataset} or asymmetric divergences such as the Renyi divergence~\cite{mansour2012multiple}.
% For Gaussian data this reduces to functionals of the covariance matrix as detailed in Example~\ref{eg:nonparametric}.
Moreover, since $C$ is a functional of the data distribution tuple $(\D_1,\ldots,\D_K)$, it can also be estimated from the unlabeled iid covariates $\{\x_{i,t}\}_{i \in [K], t \geq 1}$.
In the case of absolutely continuous distributions, Example~\ref{eg:nonparametric} shows that $C$ can be characterized through a uniform upper bound on density ratios; indeed, density ratio estimation is an extensively studied topic in its own right~\cite{kpotufe2017lipschitz,lin2021estimation,yu2021unified}.
This raises the possibility of leveraging independent advances in transfer coefficient estimation to create new data-adaptive exploration schedules for model selection among general function classes.

\end{document}